\documentclass[lettersize,journal]{IEEEtran}
\usepackage{amsmath,amsfonts}
\usepackage{algorithmic}
\usepackage{algorithm}
\usepackage{array}
\usepackage[caption=false,font=normalsize,labelfont=sf,textfont=sf]{subfig}
\usepackage{textcomp}
\usepackage{stfloats}
\usepackage{url}
\usepackage{verbatim}
\usepackage{graphicx}
\usepackage{cite}
\usepackage{fancyhdr}
\usepackage{booktabs}       
\usepackage{multirow}
\usepackage{cite}
\usepackage{enumerate}
\usepackage{ntheorem}
\newtheorem{theorem}{\bf Definition}
\newtheorem{lemma}{\bf Lemma}
\usepackage{xcolor} 
\usepackage{pifont}

\usepackage{hyperref}
\hypersetup{hypertex=true,
	colorlinks=true,
	linkcolor=red,
	anchorcolor=red,
	citecolor=red}
\begin{document}
	
	\title{Feature-Label Modal Alignment for Robust Partial Multi-Label Learning}
	
	\author{Yu~Chen, Weijun~Lv, Yue~Huang, Xiaozhao~Fang*, Jie~Wen,~\IEEEmembership{Senior Member,~IEEE}, Yong~Xu,~\IEEEmembership{Senior Member,~IEEE} and Guanbin~Li, ~\IEEEmembership{Member,~IEEE}
		\IEEEcompsocitemizethanks{
			\IEEEcompsocthanksitem Yu Chen, Weijun Lv, Yue Huang and Xiaozhao Fang are with the School of Automation, Guangdong University of Technology, Guangzhou 510006, China (e-mail: chenyu9265324@163.com, lvweijun0201@163.com, 17324004911@163.com, xzhfang168@126.com).
			\IEEEcompsocthanksitem Jie Wen and Yong Xu are with the Shenzhen Key Laboratory of Visual Object Detection and Recognition, Harbin Institute of Technology, Shenzhen 518055, China (e-mail: jiewen\_pr@126.com, yongxu@ymail.com).
			\IEEEcompsocthanksitem  Guanbin Li is with the School of Computer Science and Engineering, Sun Yat-Sen
			University, Guangzhou 510275, China (e-mail: liguanbin@mail.sysu.edu.cn).
			\IEEEcompsocthanksitem Corresponding author: Xiaozhao Fang (xzhfang168@126.com).
		}
	}
	
	\markboth{IEEE Transactions on Multimedia}%
	{Shell \MakeLowercase{\textit{et al.}}: A Sample Article Using IEEEtran.cls for IEEE Journals}
	
	\maketitle
	
	\begin{abstract}
		In partial multi-label learning (PML), each instance is associated with a set of candidate labels containing both ground-truth and noisy labels.  The presence of noisy labels disrupts the correspondence between features and labels, degrading classification performance.  To address this challenge, we propose a novel PML method based on feature-label modal alignment (PML-MA), which treats features and labels as two complementary modalities and restores their consistency through systematic alignment.  Specifically, PML-MA first employs low-rank orthogonal decomposition to generate pseudo-labels that approximate the true label distribution by filtering noisy labels.  It then aligns features and pseudo-labels through both global projection into a common subspace and local preservation of neighborhood structures.  Finally, a multi-peak class prototype learning mechanism leverages the multi-label nature where instances simultaneously belong to multiple categories, using pseudo-labels as soft membership weights to enhance discriminability.  By integrating modal alignment with prototype-guided refinement, PML-MA ensures pseudo-labels better reflect the true distribution while maintaining robustness against label noise.  Extensive experiments on both real-world and synthetic datasets demonstrate that PML-MA significantly outperforms state-of-the-art methods, achieving superior classification accuracy and noise robustness.
	\end{abstract}
	
	\section{Introduction}
	
	\IEEEPARstart{I}{n} traditional classification, each instance is assigned a single label. However, real-world objects often possess multiple properties simultaneously—an article may cover love, war, and patriotism; music may blend pop, rock, and rap. Multi-label learning \cite{trends} addresses this by allowing multiple labels per instance and has been applied to text classification \cite{ap1}, image recognition \cite{ap2}, and protein function prediction \cite{ap3}.
	
	Despite its effectiveness, multi-label learning requires high-quality annotations, which are expensive and difficult to obtain. Reliance on semi-automated tools or crowdsourced platforms inevitably introduces annotation errors, yielding candidate label sets containing both ground-truth and noisy labels. As shown in Fig.~\ref{fig_1}, the candidate set includes ground-truth labels ("lake," "cloud," "sun," "tree") alongside noisy labels ("mountain," "sea," "grass," "bird").
	
	\begin{figure}[t]
		\centering
		\includegraphics[width=\linewidth]{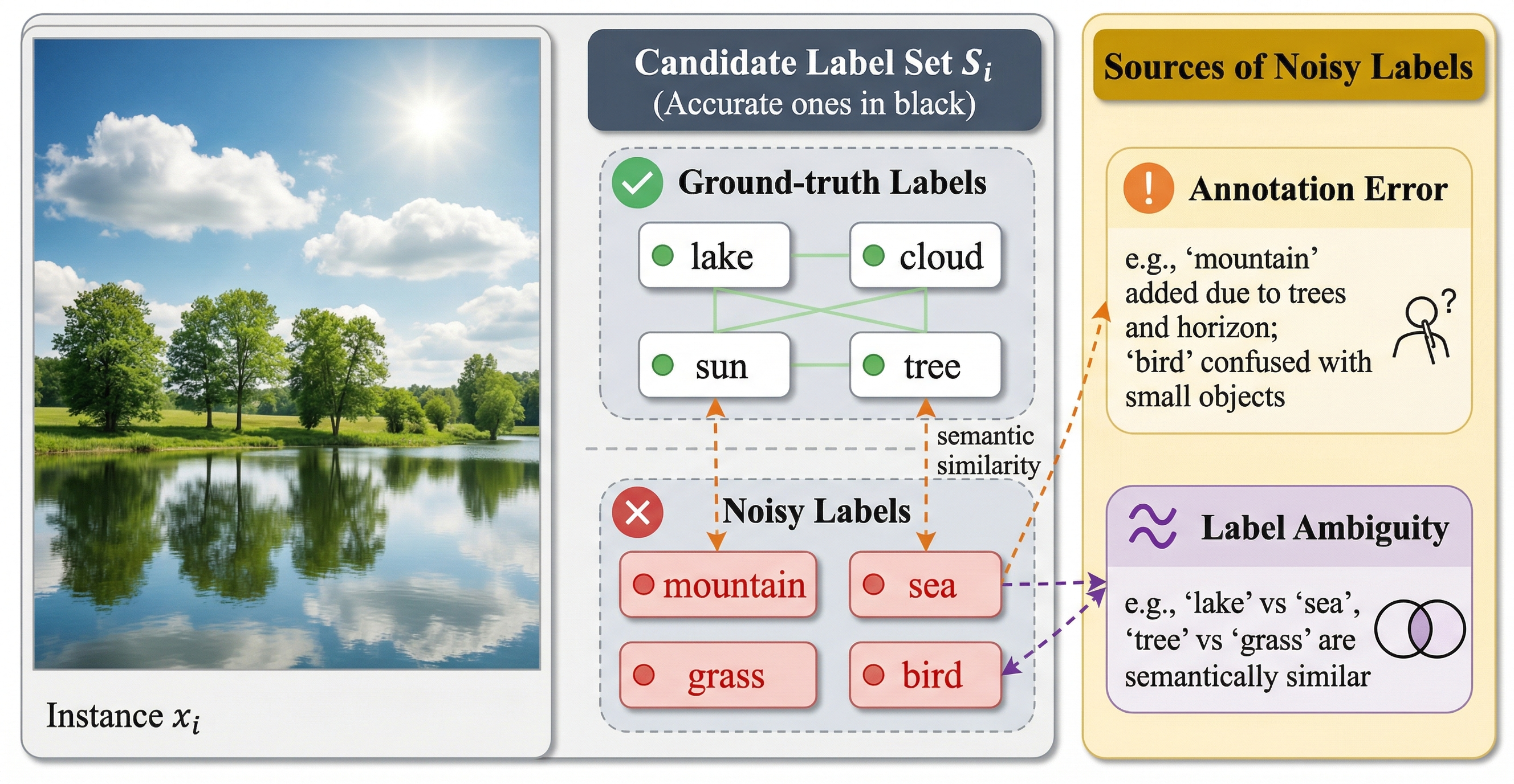}
		\caption{{An example of partial multi-label learning. The candidate label set contains both ground-truth labels (in black: "lake," "cloud," "sun," and "tree") and noisy labels (in red: "mountain," "sea," "grass," and "bird").}}
		\label{fig_1}
	\end{figure}
	
	Noisy labels arise from two primary sources. First, annotation errors occur when annotators incorrectly assign labels—in Fig.~\ref{fig_1}, "mountain" may be added due to trees and the horizon, while "bird" might be confused with small objects. Second, label ambiguity stems from semantic similarity between concepts like "lake" versus "sea" or "tree" versus "grass," leading to semantically related but incorrect labels.
	
	To address this, Xie and Huang \cite{pml-fp} introduce partial multi-label learning (PML) as a weakly supervised framework where ground-truth labels are concealed within candidate sets. Traditional multi-label algorithms such as RankSVM \cite{ranksvm2}, ML-kNN \cite{ml-knn}, and LIFT \cite{lift} lack noise identification mechanisms and perform poorly on noisy data. Subsequent PML methods fall into two-stage and end-to-end approaches \cite{lenfn}.
	
	Two-stage methods first identify ground-truth labels and then train classifiers on refined labels. For example, PARTICLE \cite{particle} employs label propagation with K-neighbor aggregation and confidence thresholding, while PENAD \cite{penad} recovers label distributions using observed logical models and shared topological structures. In contrast, end-to-end approaches jointly optimize label disambiguation and classifier training. Methods like fPML \cite{fpml} decompose candidate labels and features into low-rank matrices, PML-LRS \cite{pml-lrs} decomposes prediction model matrix into ground-truth label prediction and noise label, and PML-ND \cite{pml-nd} employs negative label correlation information to guide label propagation process to induce ground-truth labels with high credibility.
	
	Although existing methods achieve notable results, they often overlook the inherent relationship between features and labels, leading to suboptimal disambiguation under complex noise due to insufficient feature-label space alignment. {Specifically, noisy labels create a \emph{misalignment} between the feature space and label space: instances with similar features may receive dissimilar noisy label sets due to random annotation errors, and conversely, dissimilar instances may share noisy labels due to label ambiguity. This misalignment propagates through the learning pipeline, causing classifiers to learn spurious feature-label associations. Explicitly enforcing feature-label consistency provides a principled mechanism to constrain pseudo-label refinement to be coherent with the feature manifold structure, thereby filtering noise more effectively.} To address this, we propose \textbf{feature-label modal alignment}, treating features and labels as distinct yet interconnected modalities that jointly describe instance properties. Noisy labels disrupt this alignment, creating feature-label ambiguity that degrades performance. Our method explicitly aligns features and labels as complementary modalities through a systematic alignment process to restore consistency and reduce noise impact, enhancing both robustness and pseudo-label discriminability.
	
	Specifically, we leverage the subset property of ground-truth labels to perform low-rank orthogonal decomposition on candidate labels, generating pseudo-labels that approximate the true distribution. {We then project features and pseudo-labels into} a common subspace for global alignment while preserving neighborhood structures through local alignment. Finally, we learn multi-peak class prototypes that capture the multi-label characteristic where instances belong to multiple categories simultaneously. Pseudo-labels serve as soft membership weights to supervise prototype generation, with instance-prototype distances reflecting category membership degrees. This framework ensures pseudo-labels align with features and approximate ground-truth, effectively mitigating noise impact. The main contributions are summarized:
	\begin{itemize}
		\item This paper introduces the concept of modal alignment to address the PML problem, improving feature-label consistency through synergistic global and local alignment strategies, thereby enhancing classification performance.
		
		\item As a pilot study, this paper proposes a low-rank orthogonal decomposition method to reduce label ambiguity and generate more robust pseudo-labels.
		
		\item We develop a multi-peak class prototype learning mechanism that explicitly models the multi-label characteristic, enhancing pseudo-label discriminability through instance-prototype membership.
		
		\item Extensive experiments demonstrate that PML-MA significantly outperforms state-of-the-art methods with superior robustness against label noise.
	\end{itemize}
	The remainder of this article is organized as follows: Section \ref{sec2} reviews partial multi-label learning and related fields. Section \ref{sec3} presents the detailed PML-MA methodology. Section \ref{sec4} reports experimental results and analysis. Finally, Section \ref{sec5} concludes the paper.
	\section{Related Work}\label{sec2}
	This section briefly reviews the related research on partial multi-label learning, as well as the closely related fields of multi-label learning and partial label learning.
	\subsection{Multi-Label Learning}
	In multi-label learning (MLL), each example is associated with multiple labels, and various approaches have been proposed to address this challenge. One common approach is to transform MLL into multiple binary classification tasks, treating each label independently \cite{mll-tkde3}. To enhance performance, many methods have incorporated label correlations, such as pairwise and higher-order dependencies, to capture the relationships between labels. For example, classifier chains \cite{mll-tkde} and label power set \cite{mll4} methods have been used to model label co-occurrence and dependencies.
	Recent research has also integrated manifold learning with MLL to exploit the intrinsic structure of the data. Geng et al. \cite{mll5} explores the manifold structure in the label space, while Luo et al. \cite{mll6} applies low-dimensional embeddings to label information based on manifold learning and sparse feature selection. Jia et al. \cite{mll-tkde2} uses decomposed label encoding, which enables modeling alignment in an encoded label space.  
	However, traditional MLL methods assume precise labeling, which is often unrealistic. 
	\subsection{Partial Label Learning}
	Partial label learning (PLL) is a multi-class single-label task where each instance is associated with a set of candidate labels, one of which is correct and the others are noisy. The primary goal in PLL is to predict the true label from the candidate set \cite{pll-0}. A common strategy for addressing this challenge is disambiguation, which aims to recover the true label from the noisy set of candidate labels. This can be done using averaging disambiguation, which assumes that each candidate label is equally important and calculates the average output of all candidates \cite{pll-2,pll-3}, or discriminative disambiguation, where the true label is treated as a latent variable, and model parameters are iteratively refined based on an optimization criterion, such as the maximum margin \cite{pll-tkde,pll-5}. While PLL methods effectively handle noisy labels, they are typically designed for single-label prediction, making them unsuitable for multi-label tasks. 
	\subsection{Partial Multi-label Learning}
	Compared to multi-label learning (MLL) and partial label learning (PLL), partial multi-label learning (PML) is more challenging as it requires learning in imperfect environments and training high-precision classifiers. Recently, several algorithms have been developed to address PML. Currently, PML research primarily focuses on identifying ground-truth labels from the candidate label set, a process known as label disambiguation. Existing methods can be broadly categorized based on their disambiguation strategies. \textbf{(i) Low-rank and sparse decomposition approaches} separate ground-truth labels from noisy ones by exploiting structural priors, such as PML-LRS~\cite{pml-lrs} and PMFS-LRS~\cite{pmfs-lrs}. \textbf{(ii) Label propagation and enhancement methods} estimate the credibility of candidate labels through neighborhood information or distributional modeling, including PARTICLE~\cite{particle}, PENAD~\cite{penad}, PAMB~\cite{pamb}, PML-LDL~\cite{pml-ldl}, and PML-LENFN~\cite{lenfn}. \textbf{(iii) Feature-guided disambiguation approaches} leverage instance features to identify and filter noisy labels, such as fPML~\cite{fpml}, PML-NI~\cite{pml-ni}, DRAMA~\cite{drama}, and PML-PLR~\cite{pml-plr}. Among these, methods that impose correlation constraints or manifold constraints on labels using feature information have become particularly popular, including MUSER~\cite{muser}, PASAD~\cite{pasad}, PML-DNDC~\cite{pml-dndc}, and CLLFS~\cite{cllfs}. \textbf{(iv) Feature selection methods} address disambiguation by identifying discriminative features, such as PMLFS~\cite{pmlfs}, PML-FSSO~\cite{pml-fsso}, LCFS-PML~\cite{lcfs-pml}, and PMSNE~\cite{pmsne}. \textbf{(v) Correlation-based methods} exploit label correlations or cluster assignments to assist disambiguation, including PML-fp~\cite{pml-fp}, PML-SALC~\cite{pml-salc}, GLC~\cite{glc}, PML-LC~\cite{pml-lc}, and FBD-PML~\cite{fbd-pml}. \textbf{(vi) Deep learning approaches} employ neural networks to learn robust representations and reduce noise impact, such as PML-GAN~\cite{pml-gan}, PML-ED~\cite{pml-ed}, PARD~\cite{pard}, and PML-BLS~\cite{pml-bls}. 
	
	Although these methods have made significant progress, there remains 
	untapped potential in explicitly modeling the consistency between features 
	and labels, which motivates our modal alignment framework.
	
	{\subsection{Cross-View and Subspace Learning}
		Cross-view learning and subspace learning methods aim to find common representations across heterogeneous data sources. Canonical Correlation Analysis (CCA)~\cite{cca} maximizes correlations between paired views, while Partial Least Squares (PLS)~\cite{pls} seeks projections that maximize covariance. Multi-view subspace methods~\cite{multiview1,multiview2} extend these ideas by learning shared and private components across views. While our global alignment objective (Eq.~\eqref{eq4}) shares structural similarity with these methods, our framework fundamentally differs in that one modality (pseudo-labels $\mathbf{R}$) is a latent variable being jointly optimized rather than a fixed observed view. This co-optimization of projections and pseudo-labels, combined with PML-specific constraints ($0\leq r_{ij}\leq y_{ij}$), distinguishes our approach from standard cross-view learning.}
	\section{PROPOSED METHOD}\label{sec3}
	In this section, we introduce PML-MA in detail. {For clarity, the key notations used throughout this paper are summarized in Table~\ref{notation_table}.}
	\begin{table}[t]
		\centering
		{\caption{Summary of key notations.}\label{notation_table}}
		{
			\renewcommand{\arraystretch}{1.1}
			\resizebox{1\linewidth}{!}{
				\begin{tabular}{cl}
					\toprule
					Symbol & Description \\
					\hline
					$\mathbf{X}\in\mathbb{R}^{n\times d}$ & Feature matrix of $n$ instances with $d$ dimensions \\
					$\mathbf{Y}\in\{0,1\}^{n\times c}$ & Candidate label matrix with $c$ label classes \\
					$\mathbf{R}\in\mathbb{R}^{n\times c}$ & Pseudo-label matrix ($0\leq r_{ij}\leq y_{ij}$) \\
					$\mathbf{Q}\in\mathbb{R}^{c\times c}$ & Orthogonal basis matrix ($\mathbf{Q}^\top\mathbf{Q}=\mathbf{I}_c$)\\
					$\mathbf{P}_1\in\mathbb{R}^{d\times m}$ & Feature projection matrix \\
					$\mathbf{P}_2\in\mathbb{R}^{c\times m}$ & Pseudo-label projection matrix ($\mathbf{P}_2^\top\mathbf{P}_2=\mathbf{I}_m$)\\
					$\mathbf{S}\in\mathbb{R}^{n\times n}$ & Local instance similarity matrix \\
					$\mathbf{D}\in\mathbb{R}^{n\times n}$ & Scalar weight matrix with $d_{ii}=\sum_{j=1}^c r_{ij}$ \\
					$\mathbf{V}\in\mathbb{R}^{d\times c}$ & Class prototype matrix \\
					$m$ & Common subspace dimension \\
					$\lambda,\alpha,\beta,\gamma$ & Trade-off parameters \\
					\bottomrule
		\end{tabular}}}
	\end{table}
	Let $\mathbf{X}=[\mathbf{x}_1,\mathbf{x}_2,\cdots,\mathbf{x}_n]^\top\in\mathbb{R}^{n\times d}$ denotes the feature matrix of $n$ instances with $d$-dimensional features, and $\mathbf{Y}=[\mathbf{y}_1,\mathbf{y}_2,\cdots,\mathbf{y}_n]^\top\in\{0,1\}^{n\times c}$ represents the candidate label matrix with $c$ label classes. The training dataset is $\mathcal{D}=\{(\mathbf{x}_i,\mathbf{y}_i)|1\leq i\leq n\}$, where $y_{ij} = 1$ indicates the $i$-th instance is associated with the $j$-th class label, and $y_{ij} = 0$ otherwise. Each instance has a candidate label set containing noisy labels, where unrelated labels are incorrectly marked as '$1$'.
	
	The overall framework of PML-MA is illustrated in Fig.~\ref{framework}.  PML-MA consists of four synergistic components: (1) \textit{Label Low-Rank Orthogonal decomposition} ($\mathcal{L}_{LRO}$) decomposes $\mathbf{Y}$ into orthogonal basis $\mathbf{Q}$ and pseudo-labels $\mathbf{R}$ to filter noise; (2) \textit{Global Modal Alignment} ($\mathcal{L}_{GMA}$) projects features and pseudo-labels into a common subspace; (3) \textit{Local Modal Alignment} ($\mathcal{L}_{LMA}$) preserves neighborhood structures; (4) \textit{Multi-peak Class Prototype learning} ($\mathcal{L}_{MCP}$) models the multi-peak distribution of multi-label data. These components are integrated into a unified optimization framework to jointly refine pseudo-labels and train the final classifier $\mathbf{P}_2\mathbf{P}_1^\top$.
	\begin{figure*}[t]
		\centering
		\includegraphics[width=0.7\linewidth]{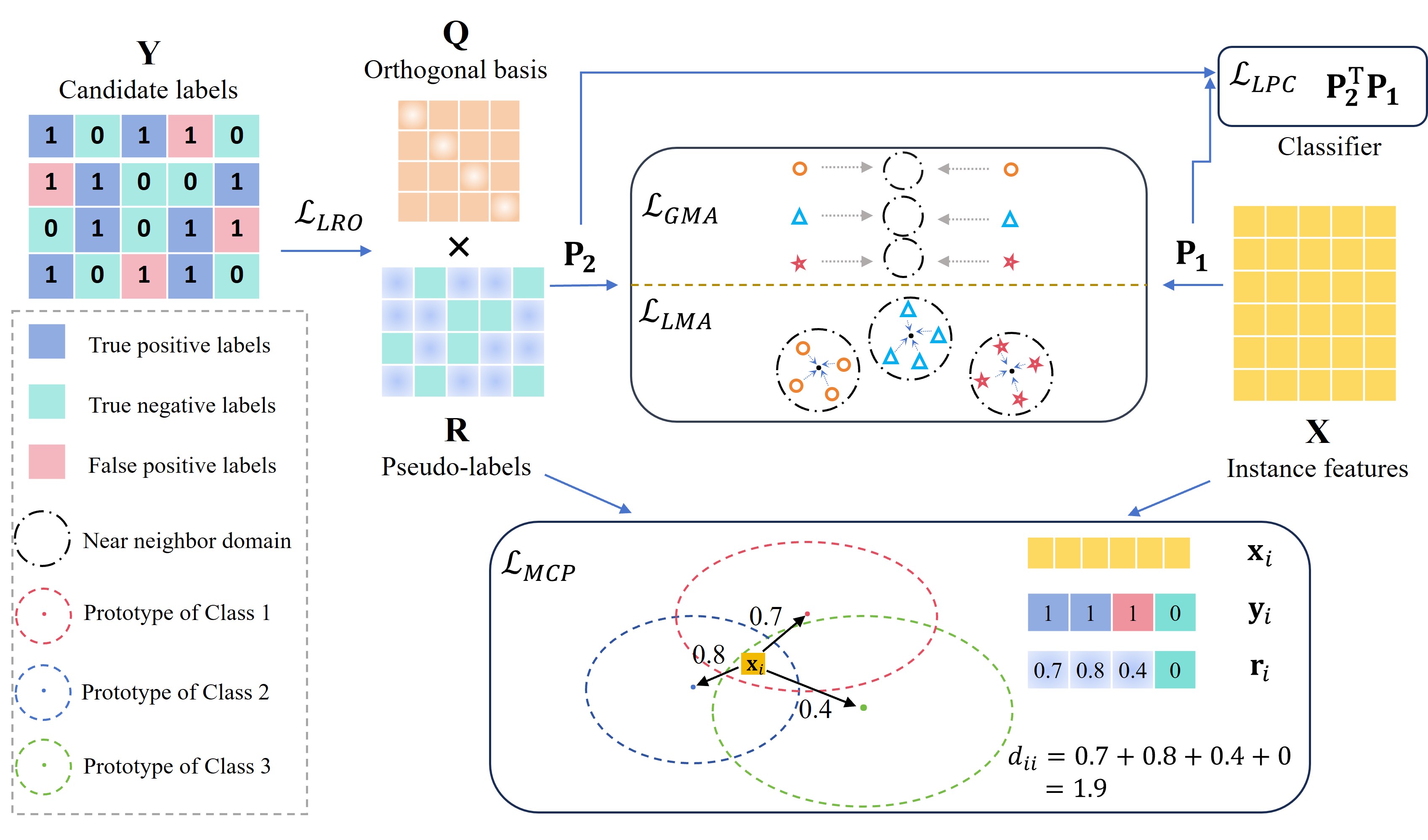}
		\caption{{Overview of the proposed PML-MA framework. The method comprises: (1) Label Low-Rank Orthogonal decomposition ($\mathcal{L}_{LRO}$) decomposes candidate labels $\mathbf{Y}$ into orthogonal basis $\mathbf{Q}$ and pseudo-labels $\mathbf{R}$; (2) Global Modal Alignment ($\mathcal{L}_{GMA}$) projects features $\mathbf{X}$ and pseudo-labels $\mathbf{R}$ into a common subspace via $\mathbf{P}_1$ and $\mathbf{P}_2$; (3) Local Modal Alignment ($\mathcal{L}_{LMA}$) preserves neighborhood structures using similarity relationships; (4) Multi-peak Class Prototype learning ($\mathcal{L}_{MCP}$) models the multi-peak distribution where instance $\mathbf{x}_i$ with pseudo-labels $\mathbf{r}_i=[0.7, 0.8, 0.4, 0]$ (yielding $d_{ii}=1.9$) is guided by multiple class prototypes. The final classifier $\mathbf{P}_2\mathbf{P}_1^\top$ is derived from the aligned projection matrices.}}
		\label{framework}
	\end{figure*}
	
	\subsection{Label Low-Rank Orthogonal Decomposition}
	In partial multi-label learning, candidate labels $\mathbf{Y}$ contain noise introduced through labeling errors. Assuming a ground-truth label matrix $\mathbf{G}\in\mathbb{R}^{n\times c}$ exists as a subset of $\mathbf{Y}$, the mainstream approach models this as:
	\begin{equation}\label{eq1}
		\min_{\mathbf{G},\mathbf{N}}~~\|\mathbf{G}\|_*+\|\mathbf{N}\|_1,\quad{s.t.}\quad \mathbf{Y}=\mathbf{G}+\mathbf{N},
	\end{equation}
	where $\mathbf{N}\in\mathbb{R}^{n\times c}$ represents the noise matrix. This assumes $\mathbf{G}$ is low-rank due to label correlations and $\mathbf{N}$ is sparse as noise occurs locally. However, unsupervised noise identification struggles with complex noise distributions.
	
	To address this, we propose orthogonal decomposition that projects $\mathbf{Y}$ into a space emphasizing signal over noise. {Unlike Eq.~\eqref{eq1} which models exact additive decomposition, we adopt a relaxed approximation framework where} we decompose $\mathbf{Y}$ into an orthogonal basis matrix and a pseudo-label matrix:
	\begin{equation}\label{eq2}
		{\mathbf{Y}\simeq\mathbf{R}\mathbf{Q}^\top,}
	\end{equation}
	where $\mathbf{Q}\in\mathbb{R}^{c\times c}$ captures label structure while reducing noise through decorrelation, and $\mathbf{R}\in\mathbb{R}^{n\times c}$ approximates the true label distribution. We impose a low-rank constraint on $\mathbf{R}$ to capture global label correlations and constrain $\mathbf{R}$ to be a subset of $\mathbf{Y}$:
	\begin{equation}\label{eq3}
		\begin{aligned}
			&\mathcal{L}_{LRO}=\min_{\mathbf{Q},\mathbf{R}}~~\|\mathbf{Y}-\mathbf{R}\mathbf{Q}^\top\|_F^2+\lambda\|\mathbf{R}\|_*\\
			&{s.t.}~~\mathbf{Q}^{\top}\mathbf{Q}=\mathbf{I}_{c},~0\leq r_{ij}\leq y_{ij}, \forall i,j,
		\end{aligned}
	\end{equation}
	where $\lambda$ controls the low-rank regularization. The orthogonality constraint ensures $\mathbf{Q}$ preserves independence across projected dimensions, while $0\leq r_{ij}\leq y_{ij}$ filters noise by restricting pseudo-labels within candidate label bounds, enabling effective discrimination between ground-truth and noise by leveraging label correlations in a controlled manner.
	
	{The orthogonality constraint $\mathbf{Q}^\top\mathbf{Q}=\mathbf{I}_c$ offers several advantages over conventional sparse decomposition (Eq.~\eqref{eq1}): (1)~\textit{Noise decorrelation}: Orthogonal projection maps correlated label noise into independent dimensions, preventing noise propagation across labels---this is particularly important in PML where annotation errors often exhibit inter-label correlations due to semantic similarity; (2)~\textit{Stable reconstruction}: Since $\mathbf{R}\mathbf{Q}^\top\mathbf{Q}=\mathbf{R}$, the decomposition preserves pseudo-label fidelity without information loss; (3)~\textit{Computational stability}: Unlike sparse optimization that requires iterative thresholding with convergence sensitivity, orthogonal projection via SVD is numerically stable. A formal proof that $\mathbf{R}^t$ approximates the ground-truth label matrix is provided in Appendix B .}
	\subsection{Feature-Label Global and Local Modal Alignment}
	To enhance pseudo-label accuracy, we align features $\mathbf{X}$ and pseudo-labels $\mathbf{R}$ as two modalities describing the same instances. We project them into a common subspace via projection matrices $\mathbf{P}_1$ and $\mathbf{P}_2$, ensuring consistency to capture intrinsic feature-label connections and guide pseudo-label refinement. Global alignment achieves overall consistency, while local alignment preserves neighborhood structures.
	Global alignment is responsible for general narrowing, while local alignment is responsible for refined narrowing. The following is a detailed introduction:
	
	\textbf{Global Modal Alignment.} We aim to minimize the distance between instance features and pseudo-labels in the common subspace, ensuring their overall consistency. Considering the different modal characteristics of feature matrix $\mathbf{X}\in\mathbb{R}^{n\times d}$ and pseudo-label matrix $\mathbf{R}\in\mathbb{R}^{n\times c}$, we introduce projection matrices $\mathbf{P}_1\in\mathbb{R}^{d\times m}$ and $\mathbf{P}_2\in\mathbb{R}^{c\times m}$ to transform $\mathbf{X}$ and $\mathbf{R}$ respectively, making them as close as possible in the $m$-dimensional common subspace (where $m\leq \min(d,c)$). The global alignment objective is:
	\begin{equation}\label{eq4}
		\mathcal{L}_{GMA}=\min_{\mathbf{P}_1,\mathbf{P}_2,\mathbf{R}}~~
		{\|\mathbf{X}\mathbf{P}_1-\mathbf{R}\mathbf{P}_2\|_F^2}.
	\end{equation}
	This objective ensures alignment of the two modalities globally by minimizing the Euclidean distance between features and pseudo-labels in the common subspace. Through global alignment, the information bias caused by inconsistency between pseudo-labels and features can be mitigated, enhancing the expressiveness and mapping ability of pseudo-labels.
	
	
	
	{\textit{Remark (Distinction from CCA/PLS):} While Eq.~\eqref{eq4} resembles the objectives of CCA (maximizing correlation $\max_{\mathbf{w}_1, \mathbf{w}_2} \frac{\mathbf{w}_1^\top \mathbf{X}_1^\top \mathbf{X}_2 \mathbf{w}_2}{\|\mathbf{X}_1\mathbf{w}_1\| \|\mathbf{X}_2\mathbf{w}_2\|}$) or PLS (maximizing covariance $\max_{\mathbf{w}_1, \mathbf{w}_2} \mathbf{w}_1^\top \mathbf{X}_1^\top \mathbf{X}_2 \mathbf{w}_2$). CCA and PLS align fixed observed views. In contrast, our pseudo-label matrix $\mathbf{R}$ is a latent variable being jointly optimized. The co-optimization of projections $\mathbf{P}_1,\mathbf{P}_2$ and $\mathbf{R}$, combined with the PML-specific constraint $0\leq r_{ij}\leq y_{ij}$, transforms the objective into a noise disambiguation mechanism.}
		
	\textbf{Local Modal Alignment.} While global alignment ensures overall consistency, it does not explicitly preserve local structures among similar instances. Therefore, we introduce local alignment to preserve neighborhood relationships in both modalities, enhancing robustness against local noise. We utilize a local instance similarity matrix $\mathbf{S}\in\mathbb{R}^{n\times n}$ to represent neighborhood similarity relationships, where $s_{ij}$ denotes the similarity between instance $i$ and instance $j$. We compute $\mathbf{S}$ using a Gaussian kernel function:
	\begin{equation}\label{eq5}
		s_{ij}=\begin{cases}{exp(\frac{-\|\mathbf{x}_i-\mathbf{x}_j\|_2^2}{\sigma^2}),~~\mathbf{x}_j\in\mathcal{N}(\mathbf{x}_i)~\text{or}~\mathbf{x}_i\in\mathcal{N}(\mathbf{x}_j),}\\{0,~~~\text{otherwise},}\end{cases}
	\end{equation}
	where $\mathcal{N}(\mathbf{x}_i)$ denotes the set of $k$ nearest neighbors of instance $\mathbf{x}_i$, and $\sigma = \frac{1}{n} \sum_{i=1}^{n} \|\mathbf{x}_i - \mathbf{x}_i^{(k)}\|_2$ controls similarity decay. The local alignment objective ensures similar instances remain close in both projected spaces, thereby maintaining neighborhood consistency:
	\begin{equation}\label{eq6}
		\mathcal{L}_{LMA}=\min_{\mathbf{P}_1,\mathbf{P}_2,\mathbf{R}}~~
		\sum_{i=1}^{n}\sum_{j=1}^{n}
		s_{ij}\|\mathbf{P}_1^\top\mathbf{x}_{i}-\mathbf{P}_2^\top \mathbf{r}_{j}\|_{2}^{2}.
	\end{equation}	
	
	{\textit{Remark (Distinction from graph regularization):} Standard graph regularization imposes within-modal smoothness (e.g., $\sum_{ij}s_{ij}\|\mathbf{r}_i-\mathbf{r}_j\|^2$), operating in a single space. In contrast, our local alignment in Eq.~\eqref{eq6} enforces \emph{cross-modal} neighborhood consistency by constraining projected features $\mathbf{P}_1^\top\mathbf{x}_i$ to be close to projected pseudo-labels $\mathbf{P}_2^\top\mathbf{r}_j$ for similar instances via two distinct projection matrices. Combined with global alignment, this creates a cross-modal framework that simultaneously addresses distribution matching and structural preservation.}	
	
	Combining the global alignment objective in Eq. \eqref{eq4} and the local alignment objective in Eq. \eqref{eq6}, and imposing an orthogonality constraint on $\mathbf{P}_2$ to eliminate redundancy and enhance robustness, we obtain:
	\begin{equation}\label{eq7}
		\mathcal{L}_{GLMA}=\mathcal{L}_{GMA}+\alpha\mathcal{L}_{LMA},
	\end{equation}
	where $\alpha$ is a trade-off parameter.
	
	\subsection{Multi-Peak Multi-Label Class Prototype Learning}
	To further improve the discriminability of pseudo-labels, we incorporate multi-label class prototype learning that explicitly models the multi-peak nature of multi-label data. Traditional k-means clustering assumes single-peak distributions, assigning each instance to exactly one cluster via:
	\begin{equation}\label{eq8}
		\mathcal{J}_{k-means}\left({\mathbf{V}}\right)=\sum_{j=1}^{c}\sum_{{x}_{i}\in\mathcal{P}_{j}}\left\|\mathbf{x}_{i}-\mathbf{v}_{j}\right\|_{2}^{2},
	\end{equation}
	where $\mathbf{V}=[\mathbf{v}_1,\cdots,\mathbf{v}_c]^\top$ are the class prototypes, $\mathbf{v}_j$ is the centroid of cluster $\mathcal{P}_j$, and $\left\|\mathbf{x}_i - \mathbf{v}_j\right\|_2^2$ represents the squared Euclidean distance between instance $\mathbf{x}_i$ and class prototype $\mathbf{v}_j$. This can be optimized using the Expectation-Maximization (EM) algorithm:
	\begin{equation}\label{eq9}
		\min_{\mathbf{V},\mathbf{L}}\sum_{i=1}^{n}\sum_{j=1}^{c}l_{ij}\|\mathbf{x}_i-\mathbf{v}_j\|_{2}^{2}, ~~ \text{s.t.}~ \mathbf{L}\in\{0,1\}^{n\times c}, ~\sum_{j=1}^{c}l_{ij}=1, ~\forall i,
	\end{equation}
	where $\mathbf{L}\in\mathbb{R}^{n\times c}$ is the class indicator matrix with $l_{ij}=1$ indicating that instance $\mathbf{x}_{i}$ belongs to cluster $\mathcal{P}_j$, and $l_{ij}=0$ otherwise. The constraint $\sum_{j=1}^{c}l_{ij}=1$ enforces single-peak distributions, which is inadequate for multi-label learning where instances simultaneously belong to multiple categories.
	
	In multi-label settings, candidate labels $\mathbf{Y}\in\{0,1\}^{n\times c}$ satisfy $\sum_{j=1}^{c}y_{ij}\geq 1$, exhibiting multi-peak distributions. We introduce a { scalar weight matrix}  $\mathbf{D}\in\mathbb{R}^{n\times n}$ with diagonal elements $d_{ii}=\sum_{j=1}^{c}y_{ij}$ representing the number of labels (peaks) associated with instance $\mathbf{x}_i$. This leads to a reformulated clustering objective in a label-weighted feature space:
	\begin{equation}\label{eq10}
		\min_{\mathbf{V}}\sum_{i=1}^{n}\left\|\mathbf{x}_id_{ii}-\sum_{j=1}^{c}y_{ij}\mathbf{v}_j\right\|_{2}^{2}.
	\end{equation}
	Here, $\mathbf{x}_i d_{ii}$ represents an aggregated feature representation scaled by the instance's label complexity, while $\sum_{j=1}^{c}y_{ij}\mathbf{v}_j$ denotes the composite prototype formed by summing all associated class centers. The optimization encourages the scaled instance to align with the sum of its relevant prototypes, effectively encoding the multi-peak structure through additive composition. For example, if $d_{ii}=3$ with $y_{i1}=y_{i2}=y_{i3}=1$, the objective minimizes $\|3\mathbf{x}_i - (\mathbf{v}_1+\mathbf{v}_2+\mathbf{v}_3)\|_2^2$, modeling the instance as a balanced combination of three class prototypes.
	
	To enhance the quality of pseudo-labels, we extend this framework to the pseudo-label matrix $\mathbf{R}$ as membership supervision information:
	\begin{equation}\label{eq11}
		\mathcal{L}_{MCP}=\min_{\mathbf{R},\mathbf{V}}\sum_{i=1}^{n}\left\|\mathbf{x}_id_{ii}-\sum_{j=1}^{c}r_{ij}\mathbf{v}_j\right\|_{2}^{2},
	\end{equation}
	where $d_{ii}=\sum_{j=1}^{c}r_{ij}$. Unlike the binary candidate labels in Eq.~\eqref{eq10}, the pseudo-labels $\mathbf{r}_i$ consist of continuous values in $[0,1]$, enabling a soft encoding of multi-peak memberships with varying intensities. This allows $d_{ii}$ to be non-integer, representing the total membership weight across all class centers. {Note that in the alternating optimization (Algorithm 1), the { scalar weight matrix} $\mathbf{D}$ is updated \emph{before} $\mathbf{R}$ in each iteration: given $\mathbf{R}^t$ from the previous iteration, we first compute $d_{ii}^{t+1}=\sum_{j=1}^c r_{ij}^t$, then optimize $\mathbf{R}^{t+1}$ with $\mathbf{D}^{t+1}$ held fixed. This sequential update avoids circular dependency and ensures stable convergence, as empirically validated in Fig.~\ref{fig_5}.} By jointly optimizing class prototypes $\mathbf{V}$ and pseudo-labels $\mathbf{R}$ through multi-peak modeling, we enforce that each instance's feature representation harmonizes with its associated label distribution in a geometrically meaningful way. This strategy enhances pseudo-label discriminability by leveraging prototype-guided learning while respecting the intrinsic multi-peak structure of multi-label data, thereby reducing sensitivity to noisy labels and improving overall robustness.

	\subsection{Linear Prediction Classifier}	
	Our ultimate goal is to induce a multi-label predictor $f: \mathbf{X} \mapsto [0, 1]^c$, which can assign an appropriate set of labels to unseen instances. We impose an orthogonality constraint on $\mathbf{P}_2$ in Eq. \eqref{eq7} and define the final linear classifier as: $\mathbf{P}_2\mathbf{P}_1^\top$, the loss function is:
	\begin{equation}\label{lpc}
		\mathcal{L}_{LPC}=\min_{\mathbf{P}_1,\mathbf{P}_2}\|\mathbf{P}_1\mathbf{P}_2^\top \|_{F}^{2}, \quad\text{s.t.}\quad\mathbf{P}_2^{\top}\mathbf{P}_2=\mathbf{I}_{m}.
	\end{equation} 
	The model's prediction output is:
	\begin{equation}
		f(\mathbf{x}_z) = \mathbf{x}_z\mathbf{P}_1\mathbf{P}_2^\top.
	\end{equation} 
	{Here, $\mathbf{P}_1\in\mathbb{R}^{d\times m}$ and $\mathbf{P}_2\in\mathbb{R}^{c\times m}$, so the classifier $\mathbf{P}_1\mathbf{P}_2^\top\in\mathbb{R}^{d\times c}$ maps a $d$-dimensional feature vector to a $c$-dimensional continuous output in $\mathbb{R}^c$. For the prediction, a threshold of 0.5 is set to generate a binary prediction $\hat{y}_j = \mathbb{I}[f(\mathbf{x}_z)_j > 0.5]$.}
	
	\subsection{Overall Objective Function}	
	\label{Overall}
	Finally, by combining Eq. \eqref{eq3}, Eq. \eqref{eq7}, and Eq. \eqref{eq11}, and adding regularization Eq. \eqref{lpc} to the classifier to control its complexity, we obtain the final overall objective function:
	\begin{equation}\label{eq12}
		\mathcal{L} = 	\mathcal{L}_{LRO}+\mathcal{L}_{GLMA}+\beta\mathcal{L}_{MCP}+{\gamma\mathcal{L}_{LPC}}
	\end{equation}
	where $\beta$ is a trade-off parameter for controlling the learning of class prototypes, $\gamma$ is a trade-off parameter controlling the complexity of the classifier. For optimization and theoretical proof, please refer to Appendix A and B.
	\section{Experiments}\label{sec4}
	\subsection{Experimental setup}
	\subsubsection{Datasets}
	To evaluate the generalization performance of PML-MA, we conduct experiments on 13 datasets: 4 real-world PML datasets \footnote{http://palm.seu.edu.cn/zhangml/} and 9 MLL datasets \footnote{http://mulan.sourceforge.net/datasets.html} . Consistent with prior PML works \cite{pamb,lenfn}, we create 26 synthetic PML configurations by adding varying levels of random label noise to the 9 MLL datasets. For real-world datasets, candidate labels are provided by ordinary annotators and ground-truth labels are verified by professional labelers \cite{huiskes2008mir,trohidis2008multi,briggs2012acoustic,snoek2006challenge}. Detailed dataset statistics are provided in Appendix C.
	
	{\subsubsection{Implementation Details}
		Five popular multi-label metrics \cite{zhang2013review}—\textit{Hamming loss}, 
		\textit{Ranking loss}, \textit{One-error}, \textit{Coverage}, and \textit{Average precision}—are used for evaluation. The first four metrics favor lower values while \textit{Average precision} favors higher values.
		All experiments use 10-fold cross-validation on MATLAB R2024a, which is a standard protocol widely adopted in PML literature~\cite{particle,pamb,lenfn} and ensures sufficient training data while providing reliable performance estimates. The same 10-fold splits are used consistently across all compared methods to ensure fair comparison. For PML-MA, hyperparameters ($\lambda$, $\alpha$, $\beta$, $\gamma$) are tuned via grid search within $\{10^{-4}, 10^{-3}, 10^{-2}, 10^{-1}, 10^{0}, 10^{1}\}$ on the training folds. The neighborhood size for similarity matrix $\mathbf{S}$ is set to $k=5$, the common subspace dimension is $m=\min(d,c)$, and the convergence tolerance is $\epsilon=10^{-3}$ with maximum iterations $T_{\max}=50$. For all baselines, we use the hyperparameter settings recommended in their original publications. For the complete code of PML-MA, please refer to \footnote{http://github.com/CcAmbiguous/PML-MA}. }
	
	\begin{table*}[!t]
		\centering
		\caption{{Comparison} of PML-MA with other state-of-the-art PML algorithms on {\itshape \textbf{{HAMMING LOSS}}} (mean$\pm$std), where the best experimental performance (the smaller the better) is shown in boldface and the suboptimal is shown with an underscore.}
		\label{HAMMING LOSS}
		\setlength{\abovecaptionskip}{0.1cm}
		\setlength{\belowcaptionskip}{0.1cm}
		\renewcommand{\arraystretch}{1}
		\resizebox{1\linewidth}{!}{{
			\begin{tabular}{ccccccccccc}
				\toprule
				{Data Set} &{avg.\#CLs} &{PML-MA} &{PML-PLR} &{FBD-PML} &{PML-ND} &{P-LENFN} &{PAMB} &{PML-NI}
				&{PARTICLE} &{PML-fp} \\ \hline
				Mirflickr
				& 3.35
				&\textbf{.169±.004}&.171±.003&\underline{.170±.003}&.185±.004&.173±.004&.171±.032&.224±.006&.174±.037&.216±.002 \\
				Music-emotion
				& 5.29
				&\textbf{.208±.004}&\underline{.210±.001}&.212±.004&.211±.005&.213±.004&.210±.003&.254±.009&.221±.004&.238±.017 \\
				Music-style
				& 6.04
				&\textbf{.114±.006}&\underline{.115±.003}&.117±.006&.116±.005&.116±.005&.115±.004&.159±.012&.125±.004&.166±.023 \\
				YeastBP
				& 18.84
				&\textbf{.033±.002}&{.034±.002}&.035±.002&.076±.003&.041±.002&.034±.007&.037±.002&\underline{.026±.012}&.040±.002 \\ \hline
				\multirow{3}{*}{emotions}
				& 3
				&\textbf{.202±.021}&\underline{.204±.019}&.218±.022&.210±.020&.226±.023&.205±.020&.312±.056&.235±.022&.377±.004 \\
				& 4
				&\textbf{.209±.015}&\underline{.215±.018}&.235±.019&.228±.020&.247±.022&.221±.022&.500±.047&.253±.022&.705±.013 \\
				& 5
				&\textbf{.240±.020}&\underline{.242±.019}&.255±.023&.248±.021&.261±.022&.243±.028&.666±.018&.404±.017&.712±.010 \\ \hline
				\multirow{3}{*}{birds}
				& 3
				&\textbf{.045±.010}&\underline{.047±.009}&.065±.008&.052±.010&.093±.008&.095±.009&.058±.010&.131±.008&.065±.007 \\
				& 4
				&\textbf{.047±.011}&\underline{.048±.009}&.068±.009&.055±.010&.093±.007&.101±.009&.061±.011&.149±.010&.060±.001 \\
				& 5
				&\textbf{.046±.010}&\underline{.048±.010}&.070±.009&.056±.009&.092±.008&.104±.009&.057±.010&.174±.011&.068±.023 \\ \hline
				\multirow{3}{*}{image}
				& 2
				&\textbf{.168±.016}&\underline{.175±.014}&.188±.013&.180±.015&.195±.016&.185±.015&.191±.010&.214±.053&.218±.008 \\
				& 3
				&\textbf{.189±.020}&\underline{.198±.018}&.210±.019&.205±.016&.215±.027&.217±.010&.227±.015&.236±.056&.219±.009 \\
				& 4
				&\textbf{.212±.020}&\underline{.225±.017}&.248±.014&.235±.015&.261±.013&.240±.011&.312±.056&.395±.054&.222±.010 \\ \hline
				\multirow{2}{*}{scene}
				& 3
				&\textbf{.110±.008}&\underline{.118±.007}&.133±.008&.125±.009&.142±.005&.143±.021&.222±.022&.145±.039&.146±.006 \\
				& 5
				&\textbf{.134±.014}&\underline{.145±.012}&.188±.016&.165±.015&.211±.017&.203±.022&.344±.053&.277±.002&.149±.010 \\ \hline
				\multirow{2}{*}{yeast}
				& 7
				&\underline{.214±.011}&.228±.009&.225±.010&.220±.008&.229±.006&\textbf{.214±.008}&.217±.017&.240±.010&.251±.005 \\
				& 9
				&\textbf{.213±.012}&{.230±.010}&.228±.009&.220±.011&.234±.012&\underline{.215±.008}&.248±.036&.217±.008&.566±.004 \\ \hline
				\multirow{2}{*}{health}
				& 7
				&.036±.003&.036±.002&.037±.003&\underline{.035±.002}&.038±.002&.044±.001&.036±.002&.051±.003&\textbf{.035±.001} \\
				& 9
				&\textbf{.037±.002}&\underline{.037±.002}&.039±.002&.038±.002&.040±.002&.046±.001&.039±.001&.055±.003&.041±.000 \\ \hline
				\multirow{3}{*}{recreation}
				& 7
				&.064±.003&.066±.002&.068±.003&\underline{.063±.002}&.084±.003&.085±.003&.065±.002&.082±.003&\textbf{.057±.001} \\
				& 9
				&.072±.004&.071±.003&.075±.003&\underline{.069±.003}&.091±.003&.102±.003&.068±.002&.079±.001&\textbf{.064±.002} \\
				& 11
				&.080±.003&.078±.003&.082±.003&\underline{.075±.003}&.096±.003&.098±.005&.070±.002&.092±.003&\textbf{.065±.002} \\ \hline
				\multirow{4}{*}{arts}
				& 5
				&.060±.002&.059±.002&.062±.003&\underline{.058±.002}&.068±.002&.068±.001&.059±.003&.078±.001&\textbf{.054±.002} \\
				& 7
				&.061±.003&.060±.002&.064±.003&\underline{.059±.002}&.071±.002&.072±.001&.059±.002&.078±.002&\textbf{.056±.003} \\
				& 9
				&\textbf{.061±.004}&\underline{.062±.003}&.068±.003&.063±.003&.075±.002&.076±.002&.062±.002&.079±.002&.065±.003 \\
				& 11
				&.065±.002&.064±.002&.070±.003&\underline{.063±.002}&.079±.002&.078±.002&.063±.001&.081±.002&\textbf{.062±.001} \\ \hline
				\multirow{4}{*}{reference}
				& 5
				&\textbf{.025±.001}&\underline{.026±.001}&.030±.001&.028±.001&.033±.001&.035±.001&.028±.001&.037±.005&.044±.000 \\
				& 7
				&\textbf{.026±.002}&\underline{.027±.001}&.032±.001&.029±.002&.036±.002&.036±.001&.030±.001&.035±.001&.045±.001 \\
				& 9
				&.031±.002&\underline{.030±.001}&.033±.002&\textbf{.029±.002}&.038±.001&.037±.001&.031±.001&.050±.002&.048±.000 \\
				& 11
				&.035±.002&\underline{.033±.001}&.036±.002&\textbf{.032±.002}&.041±.002&.037±.001&.034±.001&.049±.001&.045±.002 \\ \bottomrule
		\end{tabular}}}
	\end{table*}
	
	\begin{table*}[!h]
		\centering
		\caption{{Comparison} of PML-MA with other state-of-the-art PML algorithms on {\itshape \textbf{RANKING LOSS}} (mean$\pm$std), where the best experimental performance (the smaller the better) is shown in boldface and the suboptimal is shown with an underscore.}
		\label{RANKING LOSS}
		\setlength{\abovecaptionskip}{0.1cm}
		\setlength{\belowcaptionskip}{0.1cm}
		\renewcommand{\arraystretch}{1}
		\resizebox{1\linewidth}{!}{{
			\begin{tabular}{ccccccccccc}
				\toprule
				{Data Set} &{avg.\#CLs} &{PML-MA} &{PML-PLR} &{FBD-PML} &{PML-ND}  &{P-LENFN} &{PAMB} &{PML-NI}
				&{PARTICLE} &{PML-fp} \\ \hline
				Mirflickr
				& 3.35
				&\textbf{.108±.008}&\underline{.109±.006}&.116±.007&.113±.005&.118±.005&.112±.038&.126±.007&.127±.103&.124±.003 \\
				Music-emotion
				& 5.29
				&\underline{.235±.007}&.236±.008&.238±.008&.242±.006&.246±.009&\textbf{.234±.007}&.246±.008&.362±.014&.458±.114 \\
				Music-style
				& 6.04
				&.137±.012&\underline{.136±.009}&.138±.010&.137±.008&.138±.010&\textbf{.135±.005}&.137±.010&.221±.010&.445±.231 \\
				YeastBP
				& 18.84
				&\textbf{.195±.011}&\underline{.208±.010}&.225±.011&.215±.009&.256±.011&.230±.011&.220±.011&.404±.033&.269±.012 \\ \hline
				\multirow{3}{*}{emotions}
				& 3
				&\textbf{.156±.038}&\underline{.158±.030}&.172±.032&.165±.028&.180±.031&.160±.021&.188±.029&.250±.035&.293±.015 \\
				& 4
				&\textbf{.165±.029}&\underline{.172±.025}&.185±.026&.178±.027&.192±.027&.177±.032&.211±.027&.263±.029&.407±.014 \\
				& 5
				&\textbf{.192±.032}&\underline{.198±.028}&.235±.029&.210±.030&.249±.030&.203±.025&.276±.039&.306±.034&.419±.004 \\ \hline
				\multirow{3}{*}{birds}
				& 3
				&\textbf{.171±.027}&\underline{.172±.030}&.175±.032&.173±.035&.174±.033&.196±.042&.177±.033&.322±.033&.271±.005 \\
				& 4
				&\textbf{.192±.041}&\underline{.194±.035}&.196±.033&.195±.030&.197±.032&.204±.028&.205±.034&.326±.027&.279±.024 \\
				& 5
				&\textbf{.191±.038}&\underline{.195±.034}&.202±.035&.198±.032&.204±.036&.229±.025&.219±.036&.359±.036&.271±.010 \\ \hline
				\multirow{3}{*}{image}
				& 2
				&\textbf{.148±.018}&\underline{.155±.016}&.178±.015&.165±.017&.188±.023&.177±.022&.194±.019&.230±.060&.282±.011 \\
				& 3
				&\textbf{.175±.022}&\underline{.185±.020}&.205±.021&.195±.018&.212±.026&.217±.015&.230±.024&.261±.070&.265±.009 \\
				& 4
				&\textbf{.232±.020}&\underline{.242±.018}&.268±.017&.250±.016&.285±.017&.255±.025&.303±.010&.328±.095&.380±.017 \\ \hline
				\multirow{2}{*}{scene}
				& 3
				&\textbf{.082±.010}&\underline{.095±.009}&.125±.010&.110±.011&.138±.012&.231±.017&.176±.015&.131±.048&.245±.008 \\
				& 5
				&\textbf{.114±.015}&\underline{.138±.013}&.200±.016&.165±.015&.242±.018&.250±.016&.284±.013&.291±.130&.284±.001 \\ \hline
				\multirow{2}{*}{yeast}
				& 7
				&\textbf{.171±.015}&\underline{.173±.012}&.177±.013&.175±.010&.178±.012&.214±.008&.184±.012&.182±.010&.185±.003 \\
				& 9
				&\textbf{.179±.015}&\underline{.182±.013}&.188±.014&.185±.011&.190±.016&.211±.007&.202±.016&.189±.009&.196±.017 \\ \hline
				\multirow{2}{*}{health}
				& 7
				&\textbf{.064±.006}&\underline{.065±.005}&.070±.005&.066±.004&.076±.006&.081±.007&.096±.007&.110±.008&.067±.003 \\
				& 9
				&\textbf{.065±.006}&\underline{.068±.005}&.075±.005&.070±.004&.088±.006&.087±.008&.112±.006&.109±.007&.074±.001 \\ \hline
				\multirow{3}{*}{recreation}
				& 7
				&\textbf{.159±.011}&\underline{.163±.009}&.182±.011&.170±.010&.199±.012&.173±.008&.226±.013&.230±.011&.177±.004 \\
				& 9
				&\textbf{.171±.008}&\underline{.178±.007}&.192±.009&.185±.008&.222±.009&.195±.006&.250±.010&.238±.011&.207±.007 \\
				& 11
				&\textbf{.188±.009}&\underline{.195±.008}&.218±.010&.202±.009&.242±.012&.204±.007&.270±.014&.295±.011&.215±.002 \\ \hline
				\multirow{4}{*}{arts}
				& 5
				&\textbf{.128±.005}&\underline{.135±.006}&.145±.007&.140±.005&.149±.007&.150±.007&.173±.007&.174±.006&.157±.000 \\
				& 7
				&\textbf{.139±.010}&\underline{.145±.008}&.158±.009&.150±.007&.168±.010&.152±.007&.196±.012&.180±.008&.160±.004 \\
				& 9
				&\textbf{.152±.009}&\underline{.155±.008}&.175±.010&.162±.007&.186±.013&.159±.009&.219±.011&.185±.008&.167±.003 \\
				& 11
				&\textbf{.164±.006}&\underline{.166±.005}&.182±.008&.172±.006&.203±.007&.168±.010&.239±.008&.203±.010&.185±.005 \\ \hline
				\multirow{4}{*}{reference}
				& 5
				&\underline{.101±.006}&.105±.005&.110±.007&.108±.006&.114±.007&.112±.008&.131±.008&.147±.011&\textbf{.094±.006} \\
				& 7
				&.114±.007&.118±.006&.126±.008&.122±.007&.130±.009&\underline{.110±.008}&.152±.009&.156±.015&\textbf{.103±.004} \\
				& 9
				&\textbf{.119±.012}&\underline{.120±.010}&.135±.011&.128±.009&.141±.012&.125±.006&.165±.014&.122±.009&.119±.003 \\
				& 11
				&\textbf{.123±.011}&\underline{.128±.009}&.145±.010&.135±.008&.152±.011&.160±.013&.180±.012&.174±.012&.124±.001 \\ \bottomrule
		\end{tabular}}}
	\end{table*}
	
	\begin{table*}[htbp]
		\centering
		\caption{{Comparison} of PML-MA with other state-of-the-art PML algorithms on {\itshape \textbf{ONE-ERROR}} (mean$\pm$std), where the best experimental performance (the smaller the better) is shown in boldface and the suboptimal is shown with an underscore.}
		\label{ONE-ERROR}
		\setlength{\abovecaptionskip}{0.1cm} 
		\setlength{\belowcaptionskip}{0.1cm}
		\renewcommand{\arraystretch}{1}
		\resizebox{1\linewidth}{!}{	{
			\begin{tabular}{ccccccccccc}
				\toprule
				{Data Set} &{avg.\#CLs} &{PML-MA} &{PML-PLR} &{FBD-PML} &{PML-ND} &{P-LENFN} &{PAMB} &{PML-NI} &{PARTICLE} &{PML-fp} \\ \hline
				Mirflickr
				& 3.35 
				&\textbf{.176±.018}&\underline{.183±.014}&.196±.016&.189±.015&.282±.011&.275±.093&.296±.024&.198±.012&.292±.011 \\ \hline        
				Music-emotion             
				& 5.29
				&\textbf{.420±.012}&\underline{.424±.013}&.444±.017&.432±.015&.476±.019&.428±.032&.478±.019&.585±.036&.690±.140 \\ \hline      
				Music-style               
				& 6.04     
				&\textbf{.331±.022}&\underline{.338±.018}&.349±.019&.342±.016&.356±.019&.343±.011&.345±.018&.405±.017&.620±.254 \\ \hline
				YeastBP               
				& 18.84    
				&\textbf{.544±.033}&\underline{.548±.028}&.560±.032&.552±.030&.597±.015&.564±.023&.553±.029&.999±.000&.685±.018 \\ \hline   
				\multirow{3}{*}{emotions}     
				& 3        
				&\textbf{.235±.070}&\underline{.241±.028}&.255±.030&.246±.027&.298±.059&.244±.026&.299±.050&.307±.069&.458±.000 \\ 
				& 4        
				&\textbf{.283±.050}&\underline{.297±.032}&.318±.033&.309±.034&.329±.035&.333±.060&.355±.064&.324±.050&.542±.000 \\ 
				& 5        
				&\underline{.339±.058}&.346±.040&.358±.043&.352±.041&.419±.054&.359±.059&.448±.068&\textbf{.331±.062}&.607±.011 \\ \hline
				\multirow{3}{*}{birds}     
				& 3        
				&\textbf{.411±.097}&\underline{.412±.088}&.414±.091&.413±.089&.415±.090&.492±.095&.417±.095&.580±.045&.645±.019 \\ 
				& 4        
				&\textbf{.450±.091}&\underline{.454±.085}&.457±.084&.458±.083&.463±.092&.536±.060&.461±.060&.628±.043&.634±.016 \\ 
				& 5        
				&\textbf{.440±.077}&\underline{.462±.060}&.474±.061&.470±.058&.480±.054&.576±.052&.478±.044&.654±.091&.599±.048 \\ \hline
				\multirow{3}{*}{image}     
				& 2        
				&\textbf{.275±.037}&\underline{.295±.029}&.315±.033&.305±.030&.357±.041&.329±.038&.358±.032&.371±.106&.507±.008 \\ 
				& 3        
				&\textbf{.308±.034}&.342±.033&.360±.035&\underline{.340±.032}&.393±.052&.387±.029&.412±.040&.385±.116&.486±.017 \\ 
				& 4        
				&\textbf{.411±.048}&\underline{.420±.040}&.429±.042&.425±.041&.509±.029&.431±.039&.536±.023&.458±.114&.644±.021 \\ \hline
				\multirow{2}{*}{scene}     
				& 3        
				&\underline{.261±.016}&.268±.017&.279±.019&.274±.018&.358±.022&\textbf{.249±.023}&.418±.029&.301±.108&.603±.008  \\
				& 5        
				&\textbf{.306±.041}&\underline{.332±.030}&.365±.034&.350±.032&.542±.041&.397±.035&.591±.030&.456±.189&.622±.001  \\ \hline
				\multirow{2}{*}{yeast}     
				& 7        
				&.229±.019&.231±.020&.236±.021&.234±.019&.235±.024&\underline{.217±.031}&.242±.020&\textbf{.198±.027}&.232±.004  \\
				& 9        
				&\underline{.233±.027}&.236±.022&.241±.024&.239±.023&.245±.028&.235±.026&.264±.020&\textbf{.233±.022}&.239±.024  \\ \hline
				\multirow{2}{*}{health}     
				& 7        
				&\textbf{.253±.023}&\underline{.258±.022}&.266±.025&.262±.023&.283±.021&.396±.024&.298±.026&.576±.098&.269±.008  \\
				& 9        
				&\textbf{.267±.020}&\underline{.273±.019}&.288±.021&.277±.020&.295±.026&.429±.018&.324±.015&.545±.050&.303±.010  \\ \hline
				\multirow{3}{*}{recreation}     
				& 7        
				&\underline{.499±.018}&.510±.020&.540±.022&.525±.021&.559±.021&.626±.023&.600±.019&.713±.022&\textbf{.303±.010}  \\
				& 9        
				&\textbf{.518±.019}&\underline{.538±.018}&.560±.019&.552±.020&.638±.021&.676±.026&.585±.023&.670±.016&.660±.013  \\
				& 11        
				&\textbf{.540±.018}&\underline{.565±.017}&.600±.019&.580±.018&.673±.027&.727±.021&.626±.013&.688±.019&.685±.015  \\ \hline
				\multirow{4}{*}{arts}     
				& 5       
				&\underline{.485±.021}&.490±.018&.495±.018&.493±.018&.496±.018&.515±.018&.496±.017&.712±.011&\textbf{.482±.006}  \\
				& 7        
				&\textbf{.495±.028}&\underline{.503±.024}&.508±.026&.507±.025&.509±.019&.534±.021&.529±.021&.718±.017&.514±.006  \\
				& 9        
				&\textbf{.503±.024}&\underline{.512±.022}&.523±.024&.517±.023&.528±.033&.570±.023&.569±.025&.720±.020&.554±.014  \\
				& 11        
				&\textbf{.508±.017}&\underline{.528±.019}&.548±.021&.538±.020&.555±.022&.613±.020&.601±.019&.762±.014&.554±.014  \\ \hline
				\multirow{4}{*}{reference}     
				& 5        
				&\underline{.378±.014}&.385±.013&.400±.014&.392±.014&.410±.014&.484±.019&.413±.012&.540±.054&\textbf{.343±.004}  \\
				& 7        
				&\textbf{.383±.018}&\underline{.401±.015}&.418±.017&.410±.016&.427±.019&.506±.020&.453±.017&.715±.103&.451±.021  \\
				& 9        
				&\textbf{.394±.017}&\underline{.410±.015}&.432±.017&.420±.016&.440±.017&.521±.020&.484±.015&.645±.051&.504±.001  \\
				& 11        
				&\textbf{.410±.019}&\underline{.425±.016}&.446±.018&.436±.017&.459±.016&.532±.021&.530±.023&.678±.019&.523±.002  \\ \bottomrule
		\end{tabular}}}
	\end{table*}
	
	\begin{table*}[htbp]
		\centering
		\caption{{Comparison} of PML-MA with other state-of-the-art PML algorithms on {\itshape \textbf{COVERAGE}} (mean$\pm$std), where the best experimental performance (the smaller the better) is shown in boldface and the suboptimal is shown with an underscore.}
		\label{COVERAGE}
		\setlength{\abovecaptionskip}{0.1cm} 
		\setlength{\belowcaptionskip}{0.1cm}
		\renewcommand{\arraystretch}{1}
		\resizebox{1\linewidth}{!}{	{
			\begin{tabular}{ccccccccccc}
				\toprule		
				{Data Set} &{avg.\#CLs} &{PML-MA} &{PML-PLR} &{FBD-PML} &{PML-ND} &{P-LENFN} &{PAMB} &{PML-NI} &{PARTICLE} &{PML-fp} \\ \hline
				Mirflickr
				& 3.35 
				&\textbf{.225±.004}&\underline{.225±.005}&.229±.078&.231±.004&.226±.006&.228±.007&.227±.005&.239±.054&.226±.012 \\ \hline
				Music-emotion             
				& 5.29
				&\textbf{.402±.007}&\underline{.404±.008}&.408±.011&.406±.009&.407±.010&.410±.010&.409±.009&.510±.012&.589±.097 \\ \hline
				Music-style               
				& 6.04     
				&.204±.014&.198±.010&.201±.012&.200±.011&\underline{.196±.009}&\textbf{.195±.008}&.196±.013&.284±.013&.491±.206 \\ \hline
				YeastBP               
				& 18.84    
				&\textbf{.263±.009}&\underline{.270±.010}&.285±.012&.280±.011&.357±.016&.299±.049&.289±.009&.423±.162&.368±.015 \\ \hline
				\multirow{3}{*}{emotions}     
				& 3        
				&\textbf{.297±.030}&\underline{.298±.028}&.301±.032&.299±.030&.316±.029&.300±.036&.325±.026&.357±.045&.371±.029 \\ 
				& 4        
				&\textbf{.297±.035}&\underline{.305±.030}&.310±.031&.308±.029&.324±.032&.312±.028&.339±.031&.356±.031&.469±.009 \\ 
				& 5        
				&\textbf{.325±.033}&\underline{.330±.029}&.336±.030&.333±.028&.372±.027&.337±.028&.396±.030&.392±.039&.484±.022 \\ \hline
				\multirow{3}{*}{birds}     
				& 3        
				&\textbf{.121±.038}&\underline{.122±.037}&.126±.040&.123±.039&.124±.039&.135±.043&.129±.043&.205±.041&.212±.014 \\ 
				& 4        
				&\textbf{.138±.032}&\underline{.139±.030}&.142±.033&.141±.031&.140±.031&.146±.041&.149±.033&.200±.038&.199±.017 \\ 
				& 5        
				&\textbf{.133±.045}&\underline{.137±.040}&.141±.042&.139±.041&.145±.045&.154±.035&.151±.044&.229±.034&.208±.004 \\ \hline
				\multirow{3}{*}{image}     
				& 2        
				&\textbf{.173±.017}&\underline{.180±.016}&.185±.018&.183±.017&.204±.020&.187±.017&.208±.019&.226±.072&.283±.006 \\
				& 3        
				&\textbf{.195±.020}&\underline{.205±.018}&.215±.021&.210±.019&.223±.025&.219±.016&.237±.022&.239±.076&.267±.001 \\ 
				& 4        
				&\textbf{.242±.017}&\underline{.250±.019}&.255±.022&.253±.020&.280±.016&.257±.021&.296±.019&.287±.097&.366±.015 \\ \hline
				\multirow{2}{*}{scene}     
				& 3        
				&\textbf{.083±.009}&\underline{.095±.010}&.105±.012&.100±.011&.131±.010&.180±.022&.163±.011&.112±.048&.217±.007 \\
				& 5        
				&\textbf{.110±.013}&\underline{.120±.012}&.135±.015&.128±.014&.217±.017&.193±.017&.253±.012&.222±.121&.248±.002 \\ \hline
				\multirow{2}{*}{yeast}     
				& 7        
				&.461±.010&.463±.011&.466±.013&.465±.012&.474±.019&.468±.017&.485±.020&\textbf{.454±.014}&\underline{.461±.015} \\
				& 9        
				&\textbf{.466±.022}&.468±.018&.472±.020&.470±.019&.495±.022&.478±.021&.515±.019&\underline{.467±.017}&.469±.016 \\ \hline
				\multirow{2}{*}{health}     
				& 7        
				&\textbf{.114±.010}&\underline{.118±.008}&.123±.009&.120±.007&.138±.009&.139±.010&.166±.010&.145±.013&.125±.004 \\
				& 9        
				&\textbf{.122±.009}&\underline{.125±.007}&.130±.008&.128±.006&.153±.009&.149±.012&.185±.010&.155±.009&.132±.002 \\ \hline
				\multirow{3}{*}{recreation}     
				& 7        
				&\textbf{.210±.012}&\underline{.211±.010}&.213±.011&.212±.009&.251±.013&.214±.009&.278±.014&.261±.011&.229±.005 \\
				& 9        
				&\textbf{.222±.003}&\underline{.223±.005}&.228±.008&.225±.006&.274±.013&.234±.007&.303±.014&.272±.013&.259±.003 \\
				& 11        
				&\textbf{.239±.012}&\underline{.241±.010}&.244±.011&.242±.009&.293±.014&.245±.008&.321±.016&.335±.013&.268±.000 \\ \hline
				\multirow{4}{*}{arts}     
				& 5       
				&\textbf{.196±.008}&\underline{.200±.007}&.208±.009&.204±.008&.222±.010&.221±.009&.250±.010&.225±.006&.217±.000 \\
				& 7        
				&\textbf{.211±.016}&\underline{.215±.012}&.219±.013&.217±.011&.244±.016&.225±.010&.274±.018&.232±.011&.236±.006 \\
				& 9        
				&.227±.008&.224±.007&.226±.009&.225±.008&.265±.013&.228±.008&.299±.011&.237±.010&\textbf{.222±.011} \\
				& 11        
				&.242±.008&\underline{.238±.006}&.241±.009&.240±.008&.283±.007&.239±.011&.322±.009&.254±.010&\textbf{.237±.007} \\ \hline
				\multirow{4}{*}{reference}     
				& 5        
				&\underline{.109±.009}&.112±.007&.118±.008&.115±.006&.138±.009&.134±.008&.156±.010&.156±.011&\textbf{.093±.005} \\
				& 7        
				&\textbf{.116±.010}&\underline{.117±.008}&.120±.009&.119±.007&.155±.011&.134±.009&.177±.011&.162±.014&.121±.003 \\
				& 9        
				&\textbf{.119±.015}&\underline{.121±.010}&.128±.012&.125±.011&.165±.015&.145±.007&.190±.017&.132±.009&.132±.003 \\
				& 11        
				&\textbf{.128±.014}&\underline{.129±.012}&.132±.013&.130±.011&.177±.014&.179±.015&.205±.015&.188±.013&.148±.000 \\ \bottomrule
		\end{tabular}}}
	\end{table*}
	
	\begin{table*}[!t]
		\centering
		\caption{{Comparison} of PML-MA with other state-of-the-art PML and MLL algorithms on {\itshape \textbf{AVERAGE PRECISION}}, where the best experimental performance (the larger the better) is shown in boldface and the suboptimal is shown with an underscore.}
		\label{AVERAGE PRECISION}
		\setlength{\abovecaptionskip}{0.1cm} 
		\setlength{\belowcaptionskip}{0.1cm}
		\renewcommand{\arraystretch}{1}
		\resizebox{1\linewidth}{!}{	{
			\begin{tabular}{ccccccccccc}
				\toprule
				{Data Set} &{avg.\#CLs} &{PML-MA} &{PML-PLR} &{FBD-PML} &{PML-ND} &{P-LENFN} &{PAMB} &{PML-NI} &{PARTICLE} &{PML-fp} \\ \hline
				Mirflickr
				& 3.35 
				&\textbf{.833±.011}&\underline{.825±.010}&.815±.015&.820±.012&.798±.007&.791±.019&.786±.009&.813±.136&.792±.005 \\ \hline
				Music-emotion             
				& 5.29
				&\textbf{.633±.010}&\underline{.630±.010}&.615±.013&.628±.011&.610±.012&.626±.011&.608±.012&.506±.016&.435±.093 \\ \hline
				Music-style               
				& 6.04     
				&\textbf{.743±.017}&\underline{.742±.010}&.735±.014&.740±.012&.731±.016&.741±.007&.739±.015&.657±.012&.472±.200 \\ \hline
				YeastBP               
				& 5.93   
				&\textbf{.418±.020}&\underline{.410±.020}&.380±.023&.406±.021&.339±.015&.356±.022&.404±.022&.108±.016&.286±.009 \\ \hline
				\multirow{3}{*}{emotions}     
				& 3        
				&\textbf{.814±.043}&\underline{.812±.020}&.800±.030&.808±.025&.781±.033&.810±.017&.777±.028&.747±.035&.675±.005 \\ 
				& 4        
				&\textbf{.800±.026}&\underline{.790±.030}&.775±.035&.785±.032&.768±.028&.783±.036&.749±.034&.739±.033&.581±.003 \\ 
				& 5        
				&\textbf{.761±.028}&\underline{.755±.029}&.740±.032&.752±.030&.704±.031&.750±.028&.680±.039&.702±.037&.563±.002 \\ \hline
				\multirow{3}{*}{birds}     
				& 3        
				&\textbf{.634±.061}&\underline{.630±.058}&.620±.060&.628±.059&.627±.061&.589±.052&.617±.057&.379±.046&.424±.004 \\ 
				& 4        
				&\textbf{.597±.060}&\underline{.590±.050}&.575±.055&.585±.052&.581±.048&.564±.044&.572±.041&.419±.046&.418±.001 \\ 
				& 5        
				&\textbf{.599±.049}&.585±.030&.580±.035&\underline{.590±.028}&.570±.025&.495±.029&.564±.034&.372±.047&.457±.023 \\ \hline
				\multirow{3}{*}{image}     
				& 2        
				&\textbf{.820±.021}&\underline{.810±.022}&.790±.025&.805±.023&.780±.024&.798±.024&.770±.020&.743±.070&.672±.007 \\ 
				& 3        
				&\textbf{.795±.022}&\underline{.770±.020}&.755±.025&.765±.022&.747±.029&.748±.019&.732±.024&.725±.084&.688±.009 \\ 
				& 4        
				&\textbf{.730±.024}&\underline{.720±.025}&.700±.028&.715±.026&.671±.018&.711±.026&.653±.011&.668±.091&.573±.019 \\ \hline
				\multirow{2}{*}{scene}     
				& 3        
				&\textbf{.857±.014}&\underline{.840±.015}&.825±.018&.835±.016&.779±.012&.771±.025&.735±.016&.796±.074&.624±.001 \\
				& 5        
				&\textbf{.813±.024}&\underline{.780±.021}&.760±.024&.770±.022&.649±.025&.722±.023&.609±.018&.649±.149&.600±.000 \\ \hline
				\multirow{2}{*}{yeast}     
				& 7        
				&\underline{.758±.019}&.757±.018&.750±.019&.756±.018&.753±.018&\textbf{.768±.014}&.746±.017&.754±.013&.736±.005 \\
				& 9        
				&\underline{.752±.018}&.750±.015&.742±.017&.748±.016&.740±.017&\textbf{.753±.013}&.725±.016&.744±.011&.727±.020 \\ \hline
				\multirow{2}{*}{health}     
				& 7        
				&\textbf{.771±.017}&\underline{.768±.018}&.762±.019&.767±.018&.750±.016&.680±.016&.723±.018&.501±.045&.760±.008 \\
				& 9        
				&.753±.016&.750±.017&.740±.019&.745±.018&.733±.019&.657±.014&.693±.015&.497±.020&\underline{.736±.008} \\ \hline
				\multirow{3}{*}{recreation}     
				& 7        
				&\textbf{.600±.014}&\underline{.599±.014}&.590±.016&.595±.015&.541±.019&.519±.014&.511±.018&.418±.018&.589±.004 \\
				& 9        
				&\textbf{.580±.014}&\underline{.570±.012}&.560±.014&.565±.013&.503±.014&.470±.016&.473±.012&.453±.011&.534±.007 \\
				& 11        
				&\textbf{.557±.013}&\underline{.545±.010}&.525±.012&.535±.011&.470±.016&.438±.013&.434±.019&.397±.015&.514±.007 \\ \hline
				\multirow{4}{*}{arts}     
				& 5       
				&\textbf{.608±.014}&\underline{.602±.013}&.595±.015&.600±.014&.591±.013&.584±.012&.573±.012&.464±.010&.571±.003 \\
				& 7        
				&\textbf{.597±.017}&\underline{.590±.015}&.580±.017&.588±.016&.571±.012&.569±.013&.541±.016&.441±.014&.552±.007 \\
				& 9        
				&\textbf{.586±.018}&\underline{.583±.017}&.575±.019&.582±.017&.546±.023&.548±.015&.502±.018&.439±.014&.532±.012 \\
				& 11        
				&\textbf{.573±.008}&\underline{.560±.009}&.545±.011&.550±.010&.519±.014&.518±.016&.472±.014&.439±.014&.514±.003 \\ \hline
				\multirow{4}{*}{reference}     
				& 5        
				&\textbf{.693±.009}&.675±.011&.670±.012&\underline{.680±.010}&.664±.009&.625±.014&.653±.009&.565±.024&.629±.003 \\
				& 7        
				&\textbf{.681±.012}&.655±.013&.650±.014&\underline{.660±.013}&.643±.012&.612±.015&.614±.011&.448±.058&.612±.012 \\
				& 9        
				&\textbf{.667±.014}&.658±.016&.650±.018&\underline{.660±.015}&.627±.016&.581±.017&.586±.016&.482±.035&.569±.001 \\
				& 11        
				&\underline{.651±.014}&.640±.015&.620±.018&\textbf{.652±.016}&.609±.016&.544±.021&.551±.017&.481±.023&.549±.002 \\ \bottomrule
		\end{tabular}}}
	\end{table*}
	
	\subsubsection{Comparing Approaches} 
	In this paper, the performance of PML-MA is compared with {eight} state-of-the-art PML baselines, including {PML-PLR (IJCAI'25)~\cite{pml-plr}, FBD-PML (Neural Networks'25)~\cite{fbd-pml}, PML-ND (TMM'24)~\cite{pml-nd}, P-LENFN (ACM MM'24) \cite{lenfn}, PAMB (TPAMI'23)~\cite{pamb}, PML-NI (TPAMI'22)~\cite{pml-ni}, PARTIAL (TPAMI'22)~\cite{particle}, PML-fp (AAAI'18)~\cite{pml-fp}.}		
	\subsection{Experimental Results}
	\begin{figure*}[t]
		\vspace{-0.1cm}
		\centering
		\hspace{1cm}
		\subfloat[hamming~loss]{\includegraphics[width = 0.3\textwidth]{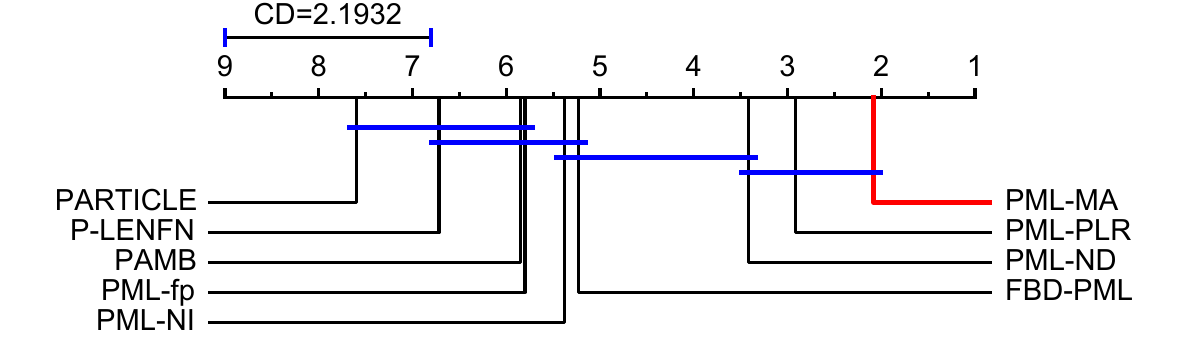}}
		\hspace{2cm}
		\subfloat[ranking~loss]{\includegraphics[width = 0.3\textwidth]{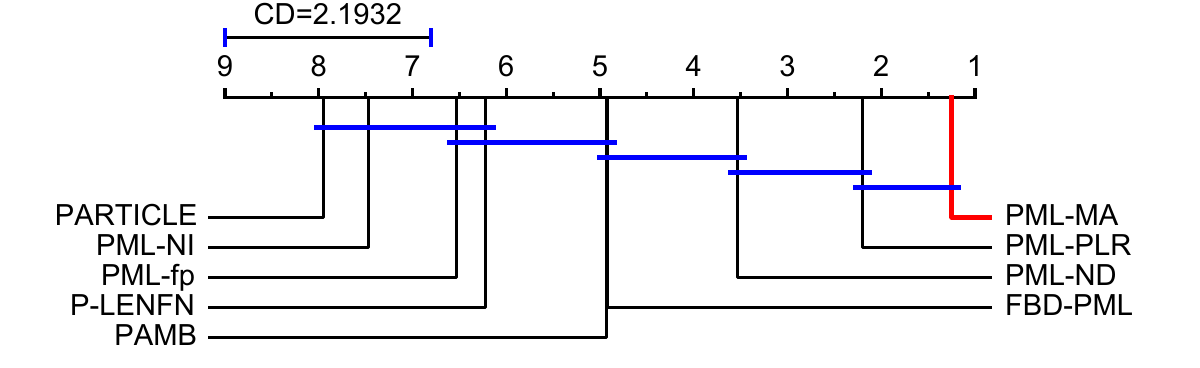}}
		\newline
		\centering
		\vspace{-0.1cm}
		\subfloat[one~error]{\includegraphics[width = 0.3\textwidth]{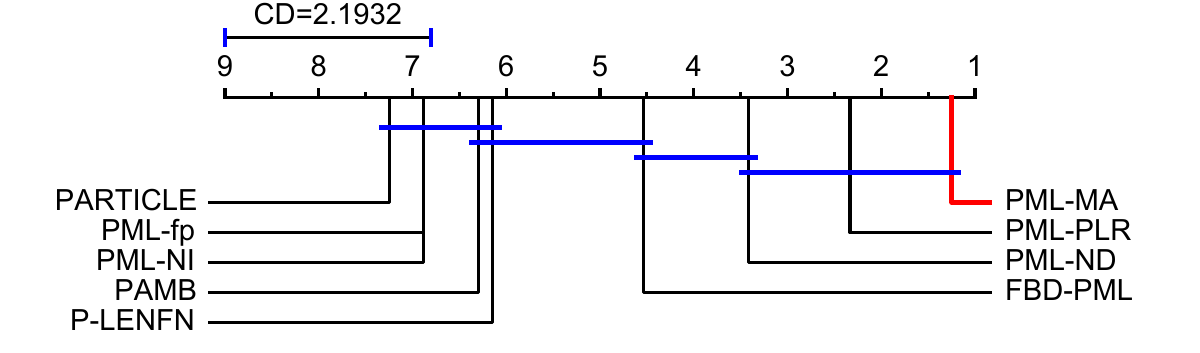}}
		\hfill
		\subfloat[coverage]{\includegraphics[width = 0.3\textwidth]{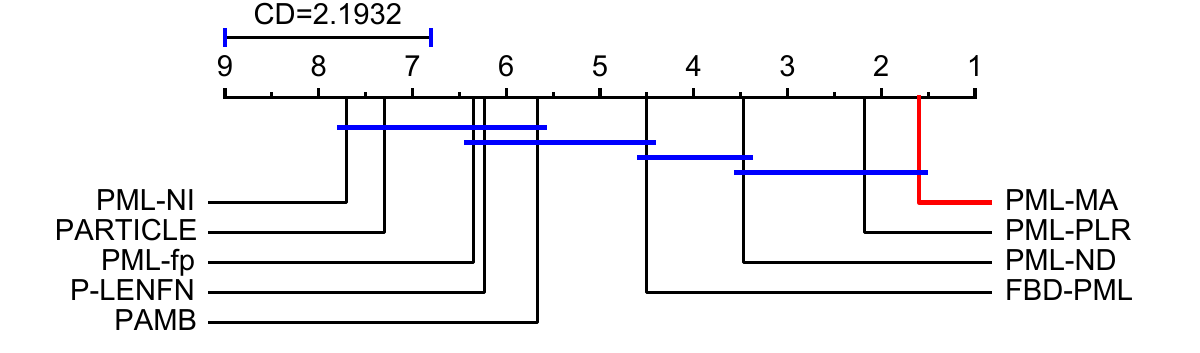}}
		\hfill
		\subfloat[average~precision]{\includegraphics[width = 0.3\textwidth]{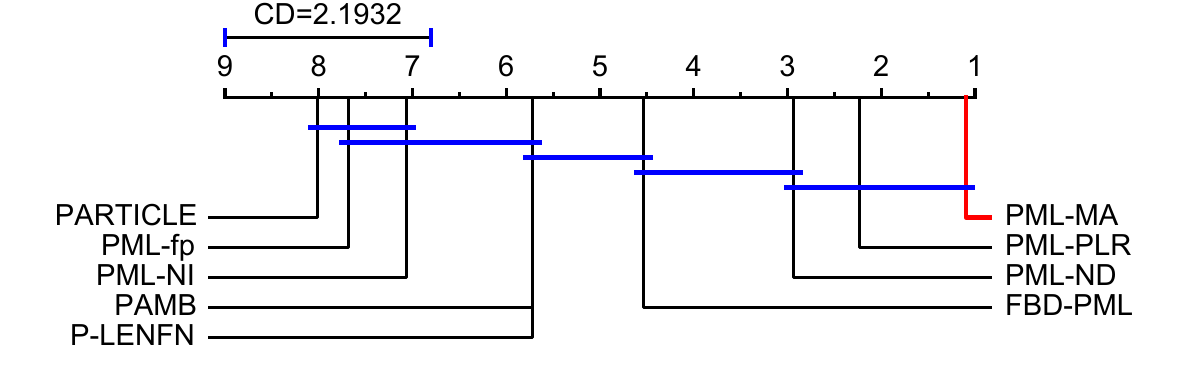}}
		\caption{{Results of PML-MA against other approaches with the Nemenyi test(CD = 2.1934 at 0.05 significance level).}}
		\label{fig_2}
	\end{figure*}
	{Tables \ref{HAMMING LOSS} to \ref{AVERAGE PRECISION} summarize the experimental results across all compared methods using five evaluation metrics. For each dataset and metric, the best performance is highlighted in boldface and the second-best is underlined. To rigorously validate the statistical significance of improvements, we conduct pairwise Wilcoxon signed-rank tests between PML-MA and each baseline across all 30 dataset configurations. The results confirm that PML-MA is significantly better than all baselines ($p<0.05$) on every metric. Due to space constraints, the average rankings of all methods across the five metrics are provided in Appendix C.}
	
	Additionally, the \textit{Friedman test} \cite{demvsar2006statistical} is applied to assess whether significant performance differences exist among all compared methods. The test evaluates the Friedman statistic $F_F$ \footnote{ The critical value can be obtained by running the Matlab command \texttt{icdf('F',1-\(\xi\), k-1, (k-1)*(N-1))} where \(\xi\) denotes the significance level.} based on the average ranks of methods across $N=30$ datasets. Table \ref{FF} presents the Friedman statistics for each evaluation metric along with the critical values at 0.05 significance level. The results clearly reject the null hypothesis across all metrics, confirming statistically significant performance differences among the compared methods.
	
	\begin{table}[htbp]
		\centering
		\vspace{-0.1cm}
		{\caption{Summary of FRIEDMAN STATISTICS $\ F_F$ FOR EACH EVALUATION METRIC AT 0.05 SIGNIFICANCE LEVEL \\
			(\# APPROACHES $k$ = 9, DATA SETS $N$ = 30).}\label{FF}}
		{
		\begin{tabular}{llc}
			\toprule
			Evaluation metric & $\ F_F$ & Critical value \\\hline
			Hamming Loss & 23.1348  &   \\
			Ranking Loss & 71.4812  &   \\
			One Error & 54.1232  & 1.9785 \\
			Coverage & 51.7904   &   \\
			Average precision & 128.4977  & \\ \bottomrule
		\end{tabular}} 
		\vspace{-0.1cm}
	\end{table}%
	
	We apply the \textit{Nemenyi test} \cite{demvsar2006statistical} as post-hoc analysis to assess whether PML-MA significantly outperforms other methods. Using PML-MA as the control method, we evaluate the average rank difference across all datasets. A performance difference is considered significant if the rank difference exceeds the \textit{critical difference} (CD) \cite{demvsar2006statistical}.
	Fig. \ref{fig_2} shows the CD plot of the five evaluation indicators, with the average ranking of each method displayed along the axis, the higher the ranking the further to the right. Where PML-MA as the control method is highlighted with a red line, if the average ranking difference between the comparison method and the control method is within CD, the bold blue line is connected. Otherwise, these methods are considered significantly different in performance compared to PML-MA.
	
	{In each diagram, methods are ranked along the axis in descending order from left to right. Methods connected to PML-MA by thick blue lines have performance differences within the CD threshold, indicating comparable performance. Otherwise, the difference is statistically significant. Key observations include:}
	\begin{itemize}
		\item {Across 150 experimental cases (30 datasets $\times$ 5 metrics), PML-MA achieves the best or second-best performance in 133 cases (88.67\%), and attains the lowest (best) average rank across all five metrics, demonstrating that feature-label modal alignment effectively improves label-feature consistency, mitigates noisy labels, and enhances classification performance. The Wilcoxon signed-rank tests further confirm that the superiority of PML-MA over every baseline is statistically significant ($p<0.05$) on all metrics.}
		
		\item {\textit{Comparison with recent SOTA methods.} Among all baselines, PML-PLR and PML-ND are the most competitive, achieving overall average ranks of 2.37 and 3.35, respectively. PML-PLR reconstructs pseudo-labels via instance-level correlations, and PML-ND leverages negative label information for disambiguation, both representing recent advances in PML. Nevertheless, PML-MA consistently outperforms them, particularly on Ranking loss (1.25 vs.\ 2.20 and 3.53) and Average precision (1.10 vs.\ 2.23 and 2.93). This suggests that explicitly aligning features and labels as complementary modalities captures richer structural information than instance-level or negative-label-based strategies alone.}
		
		\item {\textit{Comparison with prototype-based methods.} Both FBD-PML and PML-fp employ class prototype learning mechanisms similar in spirit to our multi-peak prototype component $\mathcal{L}_{MCP}$. However, PML-fp uses single-peak feature prototypes that assign each instance to only one dominant category, ignoring the multi-label nature of the data; FBD-PML connects feature and label spaces via fuzzy confidence scores but does not explicitly model multi-peak distributions. In contrast, our multi-peak prototype learning uses continuous pseudo-labels as soft membership weights, with a non-integer degree $d_{ii}=\sum_j r_{ij}$ that naturally encodes multi-label complexity. As a result, PML-MA outperforms FBD-PML (avg.\ rank 4.74) and PML-fp (avg.\ rank 6.65) by a substantial margin, confirming that multi-peak prototype modeling combined with modal alignment yields a more effective disambiguation framework.}
		
		\item {\textit{Robustness across noise levels.} On synthetic datasets where noise levels vary (avg.\#CLs from 2 to 11), PML-MA maintains stable top-tier performance even under severe noise, whereas methods like PARTICLE and PML-NI degrade rapidly. This robustness stems from the synergistic effect of orthogonal decomposition filtering noise at the label level and global-local alignment anchoring pseudo-labels to the feature manifold.}
	\end{itemize}
	\begin{figure*}[t]
		\centering
		\subfloat[$\mathbf{\lambda}$]{\includegraphics[width = 0.3\textwidth]{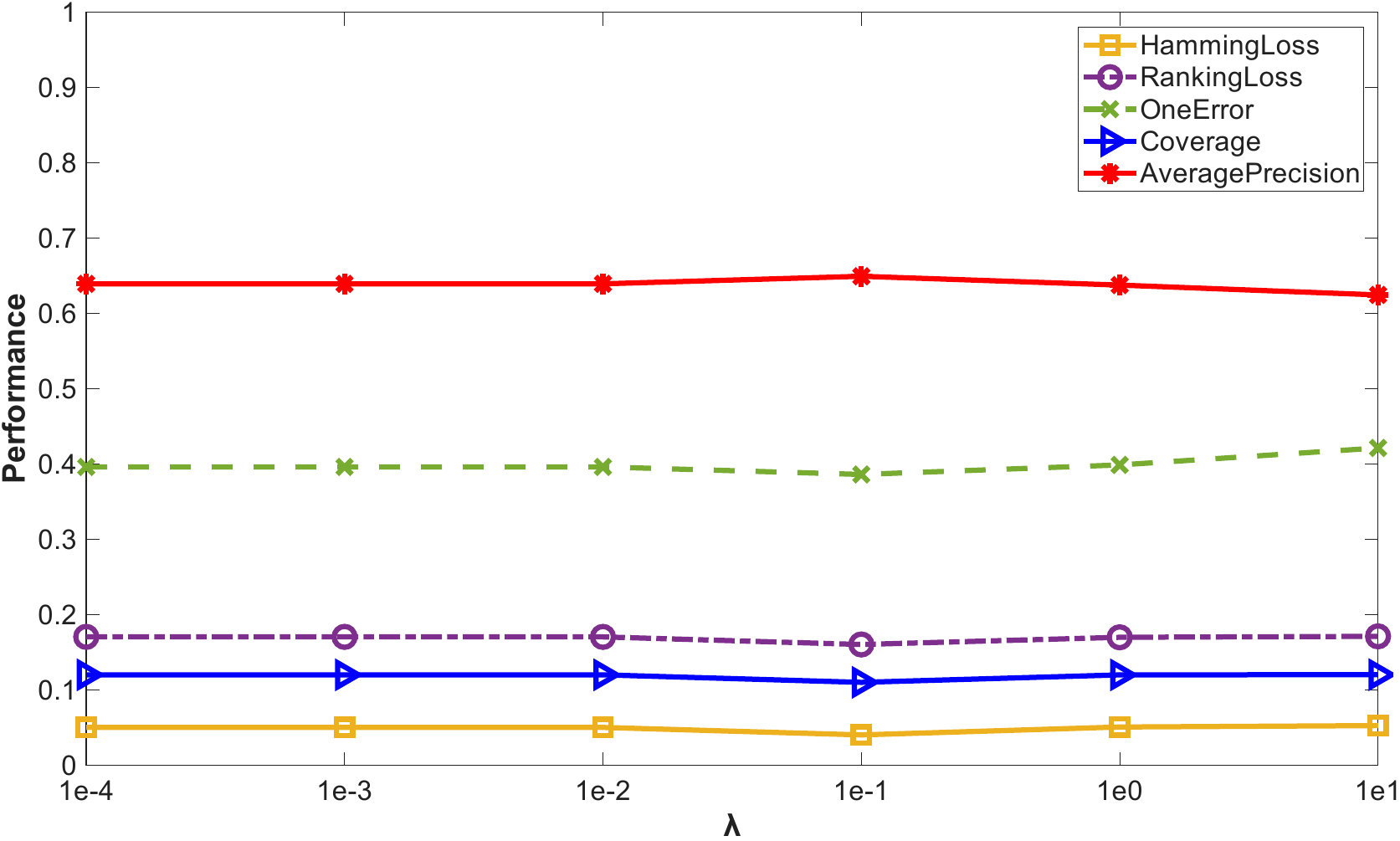}}
		\hspace{1cm}
		\subfloat[$\mathbf{\alpha}$]{\includegraphics[width = 0.3\textwidth]{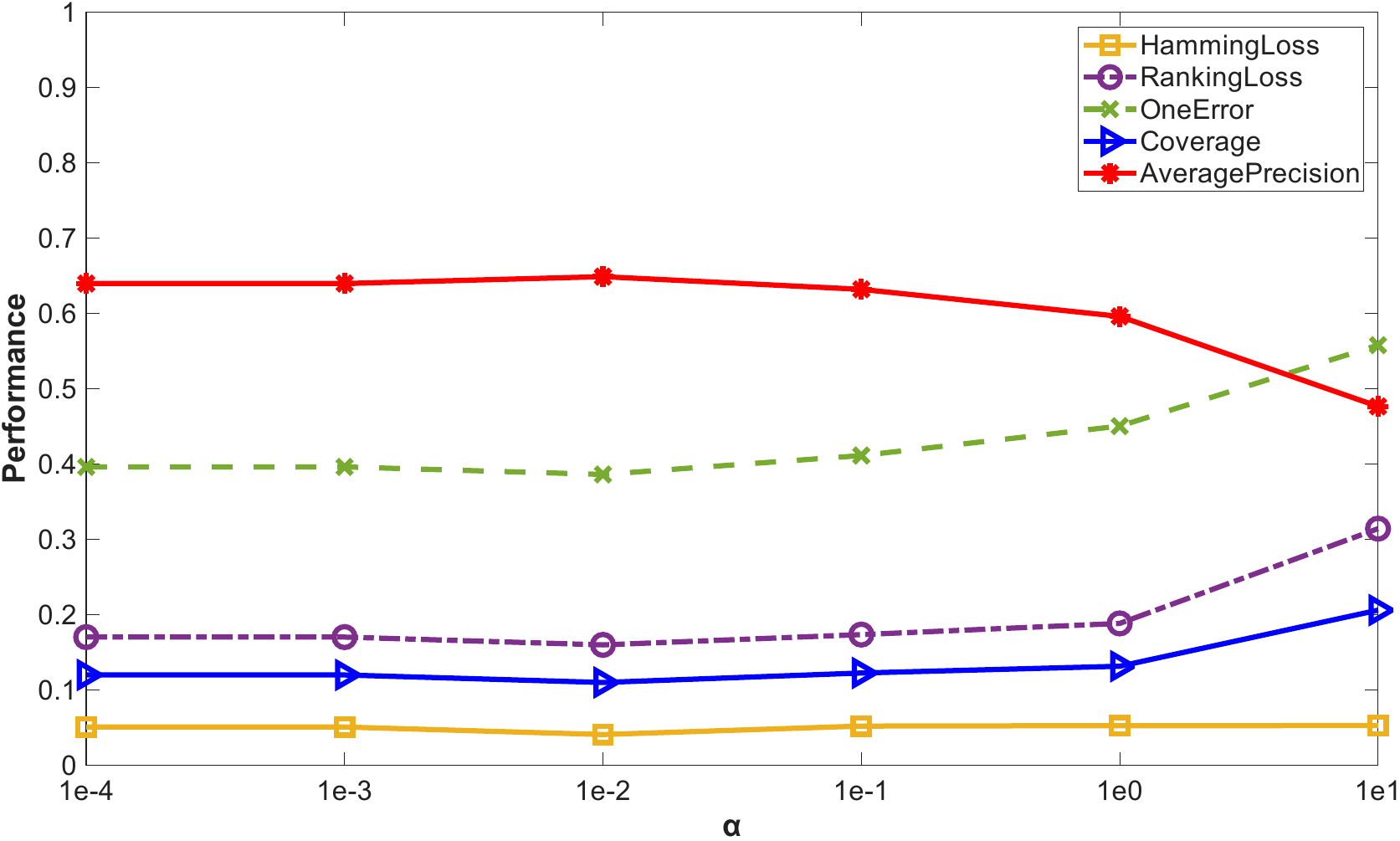}}
		\\
		\subfloat[$\mathbf{\beta}$]{\includegraphics[width = 0.3\textwidth]{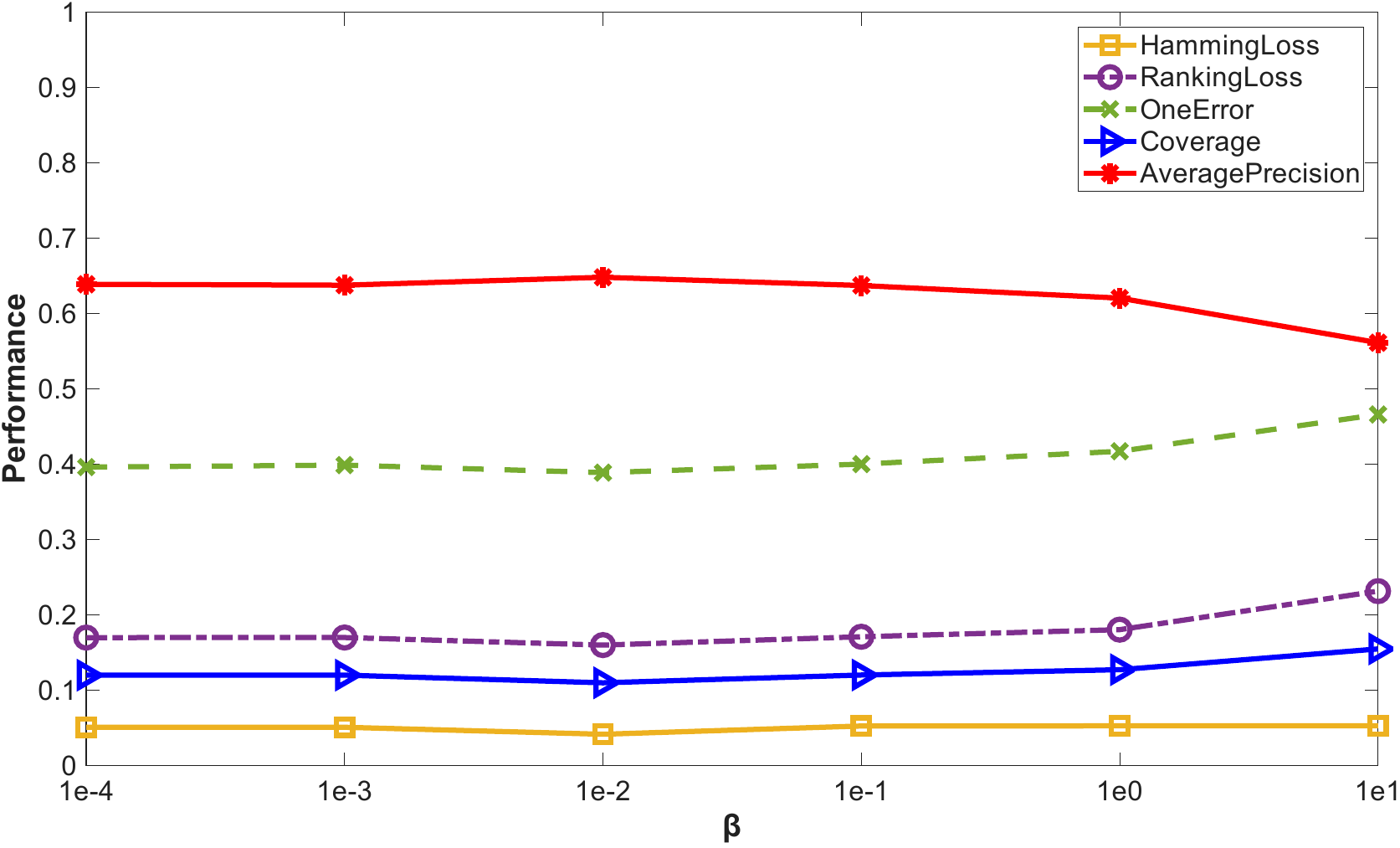}}
		\hspace{1cm}
		\subfloat[$\mathbf{\gamma}$]{\includegraphics[width = 0.3\textwidth]{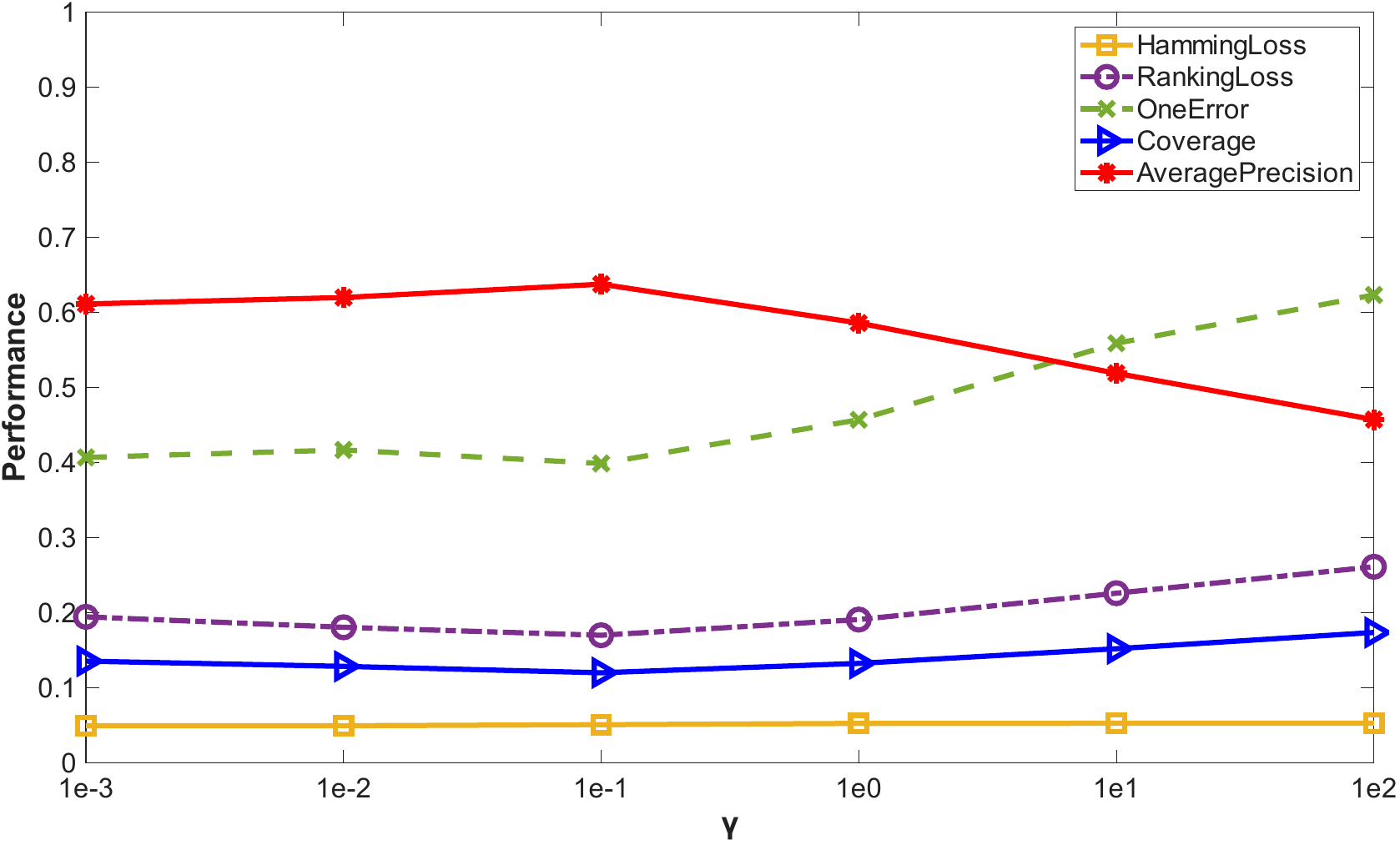}}
		\caption{Results of PML-MA with varying values of trade-off parameters on \textit{birds} dataset.}
		\label{fig_3}
	\end{figure*}
	\subsection{Parameter Analysis}
	We conduct parametric sensitivity experiments on the $birds$ dataset ($avg.\#CLS = 3$) to assess each parameter's influence. We vary four parameters ($\lambda$, $\alpha$, $\beta$, $\gamma$) within $\{10^{-4}, 10^{-3}, 10^{-2}, 10^{-1}, 10^{0}, 10^{1}\}$ while keeping others fixed. Fig. \ref{fig_3} shows that all parameters exhibit optimal values at intermediate ranges. Key findings include:
	\begin{itemize}
		\item $\lambda$ controls the low-rank constraint. Moderate values improve performance by leveraging the low-rank nature of ground-truth labels, while excessive values degrade performance, indicating a lower bound on this property.
		\item $\alpha$ regulates local alignment strength. Excessively high values cause overemphasis on local information at the expense of global consistency. {Empirically, for datasets with strong local manifold structures (e.g., image and scene datasets where neighborhood similarity is reliable), moderately larger $\alpha$ values (around $10^{-1}$ to $10^{0}$) are preferred. For datasets with weaker local structures or higher noise levels, smaller $\alpha$ values prioritize global alignment for more robust disambiguation.}
		\item $\beta$ governs class prototype learning. Overly high values lead to excessive reliance on ambiguous clustering information. {In practice, $\beta$ around $10^{-2}$ to $10^{-1}$ yields stable performance, as this range balances prototype-guided refinement with the other objectives without over-constraining pseudo-labels to potentially noisy cluster assignments.}
		\item $\gamma$ controls classifier regularization. Too small values cause overfitting, while too large values reduce model expressiveness. {We observe that $\gamma$ in the range $10^{-2}$ to $10^{0}$ generally performs well; the optimal value tends to be slightly larger on high-dimensional datasets.}
	\end{itemize}
	\begin{figure}[t]
		\centering
		\includegraphics[width=0.6\linewidth]{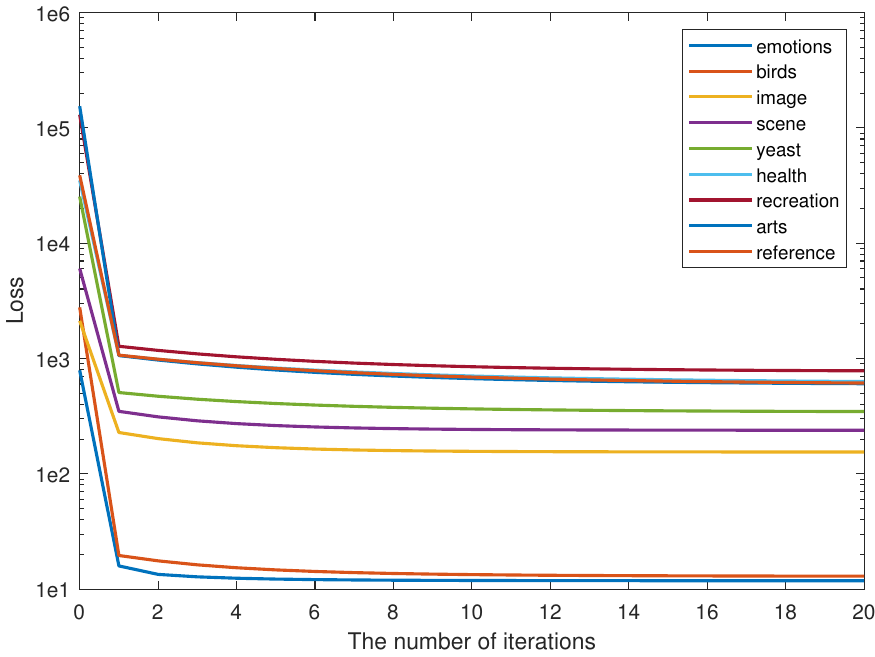}
		\caption{The convergence curves of PML-MA on the synthetic datasets.}
		\label{fig_5}
	\end{figure}
	\begin{table}[t]
		\centering
		\caption{Ablation study results. Comparison of performance with and without each module, in metric \textbf{Average precision}. The best experimental performance is shown in boldface.}
		\label{ablation}
		\renewcommand{\arraystretch}{1.2}
		\resizebox{1\linewidth}{!}{
			\begin{tabular}{l|cccc}
				\toprule
				\multicolumn{1}{l|}{Component} & \multicolumn{4}{c}{Configuration and accordingly performance} \\ \hline
				$\mathcal{L}_{LRO}$ & \ding{51} & \ding{55} & \ding{51} & \ding{51} \\ 
				$\mathcal{L}_{GLMA}$ & \ding{51} & \ding{51} & \ding{55} & \ding{51} \\ 
				$\mathcal{L}_{MCP}$ & \ding{51} & \ding{51} & \ding{51}  & \ding{55} \\ \cline{1-5}
				emotions (avg.\#CLs=4)  & \textbf{.800$\pm$.026} & .793$\pm$.020 & {.785$\pm$.013}  & .796$\pm$.023 \\ 
				birds (avg.\#CLs=4) & \textbf{.597$\pm$.060} & .590$\pm$.040 & {.584$\pm$.048}  & .595$\pm$.051 \\ 
				scene (avg.\#CLs=5) & \textbf{.813$\pm$.024} & .811$\pm$.020 & {.807$\pm$.012}  & .810$\pm$.028 \\
				yeast (avg.\#CLs=7) & \textbf{.758$\pm$.019} & .752$\pm$.022 & {.750$\pm$.018}  & .753$\pm$.014\\ \bottomrule
		\end{tabular}}
	\end{table}
	\subsection{Convergence Analysis}
	To verify convergence, we plot the objective function values of PML-MA 
	over iterations for nine datasets: $emotions$, $birds$, $image$, $scene$, 
	$yeast$, $health$, $recreation$, $arts$, and $reference$. As shown in 
	Fig.~\ref{fig_5} (with logarithmic y-axis), the loss decreases rapidly 
	and stabilizes within approximately 10 iterations across all datasets, 
	demonstrating fast convergence and confirming the efficiency of our method.
	\subsection{Ablation}
	To evaluate each module's effectiveness, we conduct ablation studies on four benchmark datasets ($emotions$, $birds$, $scene$, and $yeast$). We test the full model and three variants, each with one key component removed: label low-rank orthogonal decomposition ($\mathcal{L}_{LRO}$), global and local modal alignment ($\mathcal{L}_{GLMA}$), or multi-peak class prototype learning ($\mathcal{L}_{MCP}$).
	
	Table \ref{ablation} shows all three modules positively impact performance, with {$\mathcal{L}_{GLMA}$ (Global and Local Modal Alignment)} demonstrating the most significant contribution through improved noise reduction and feature-label alignment. These results confirm each module's effectiveness in enhancing partial multi-label learning.
	
	Additionally, we compare our low-rank orthogonal decomposition (\textbf{LRO}) with conventional low-rank sparse decomposition (\textbf{LRS}) on the $yeast$ and $reference$ datasets. We train models using ground-truth labels (\textbf{GTL}), labels from Eq. \eqref{eq1} (\textbf{LRS}), and pseudo-labels from Eq. \eqref{eq3} (\textbf{LRO}). As shown in Fig. \ref{fig_6}, \textbf{LRO} achieves higher consistency with \textbf{GTL}. While both methods perform comparably under low noise, \textbf{LRO} exhibits superior robustness as noise increases---the sparse assumption of \textbf{LRS} struggles with high noise, whereas \textbf{LRO}'s orthogonal decomposition maintains stability and label reliability.
	
	\begin{figure}[htbp]
		\centering
		
		\subfloat[yeast]{\includegraphics[width = 0.3\textwidth]{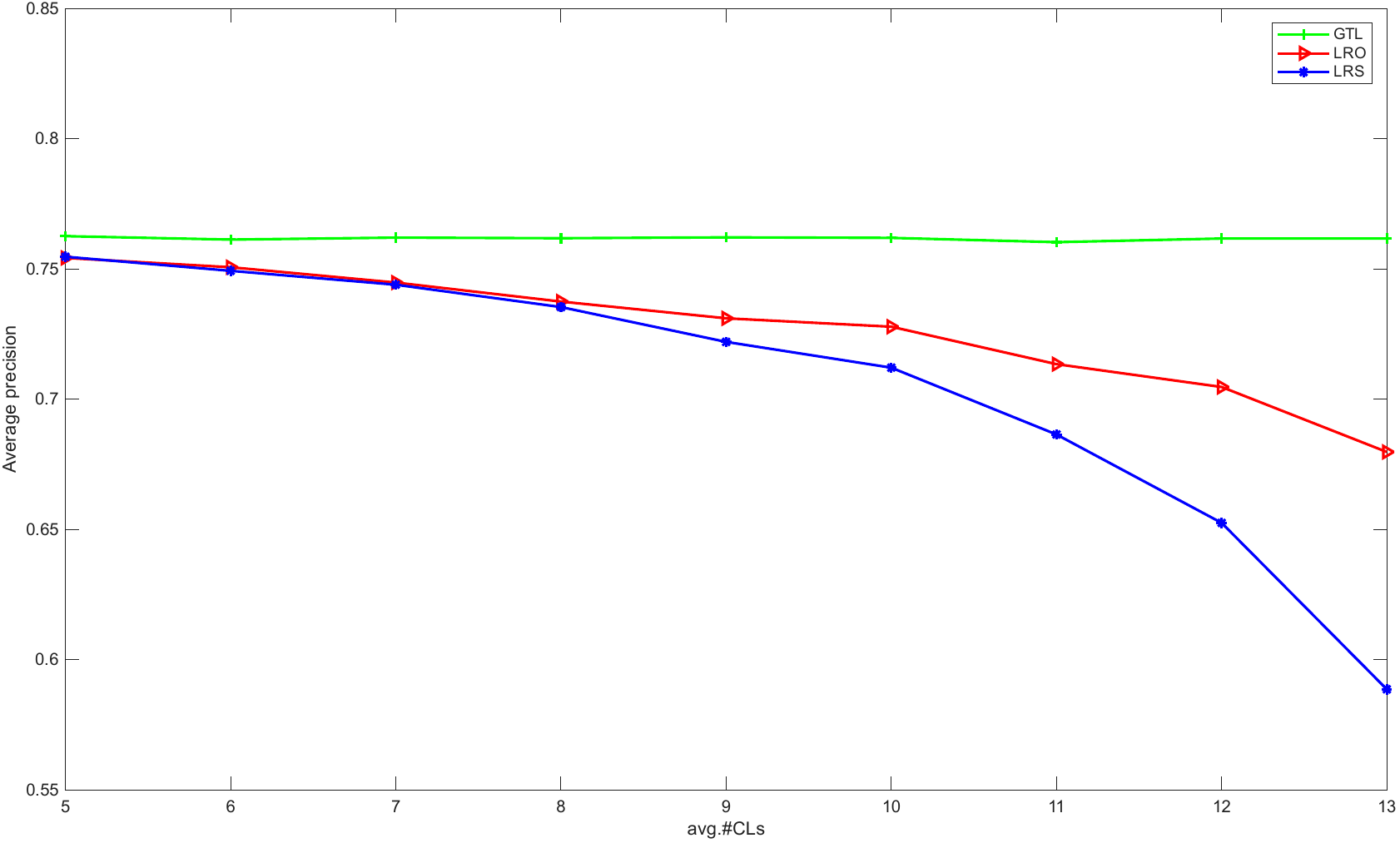}}
		\hfill
		\subfloat[reference]{\includegraphics[width = 0.3\textwidth]{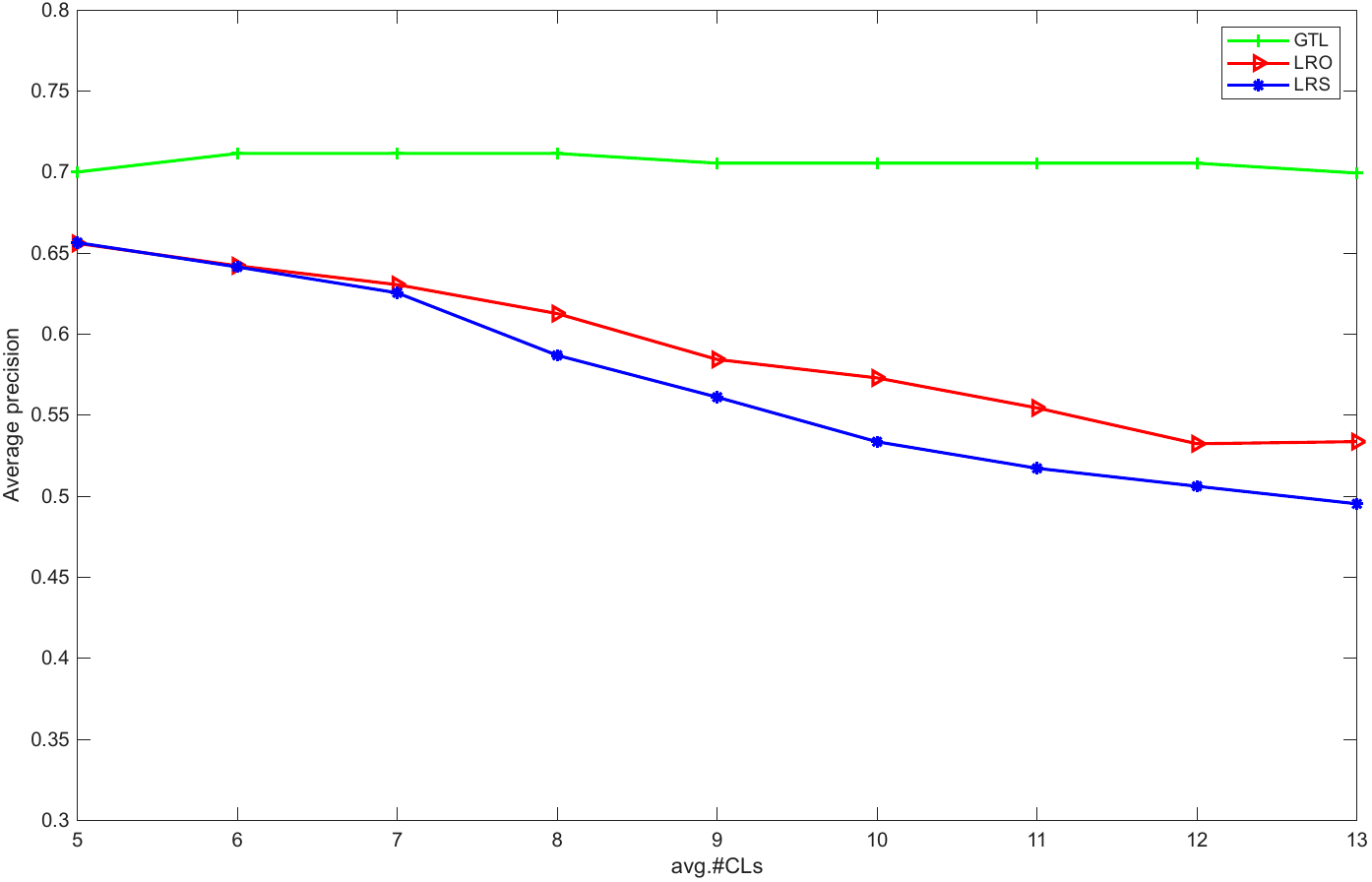}}
		\caption{{Comparative experiment of low rank orthogonal decomposition and low rank sparse decomposition on \textit{yeast} and \textit{reference} datasets.}}
		\label{fig_6}
	\end{figure}

	\section{Conclusion}\label{sec5}
	
	This paper proposes PML-MA, a partial multi-label learning method based 
	on feature-label modal alignment that effectively mitigates noisy label 
	impact. Our method integrates three key components: low-rank orthogonal 
	decomposition for pseudo-label generation, global-local modal alignment 
	for feature-label consistency, and multi-peak prototype learning with 
	soft membership supervision. Experiments on 30 dataset configurations show that PML-MA 
	significantly outperforms state-of-the-art methods with strong robustness. 
	{Despite these promising results, several limitations of the current framework motivate future research: (1)~PML-MA involves multiple trade-off parameters ($\lambda$, $\alpha$, $\beta$, $\gamma$) that require grid-search tuning for each dataset, incurring additional computational cost and limiting out-of-the-box applicability. Future work will explore \textit{adaptive parameter mechanisms} that automatically calibrate these parameters based on estimated noise levels and data characteristics. (2)~The linear projection matrices $\mathbf{P}_1$ and $\mathbf{P}_2$ have limited capacity to capture complex nonlinear feature-label relationships. Future work will investigate \textit{deep neural integration}, extending modal alignment to nonlinear architectures for end-to-end representation learning. (3)~PML-MA operates on a single feature modality, yet real-world multimedia data often involves heterogeneous sources (e.g., visual, textual, and audio). Future work will develop \textit{multi-view extensions} that align pseudo-labels with multiple feature modalities simultaneously. (4)~The per-iteration complexity of $\mathcal{O}(n^2d)$ poses scalability challenges for large-scale datasets. Future work will explore \textit{mini-batch optimization and approximate SVD techniques} to reduce computational overhead.}
	
	\bibliographystyle{IEEEtran}
	\bibliography{reference.bib}

@article{trends,
	title={The emerging trends of multi-label learning},
	author={Liu, Weiwei and Wang, Haobo and Shen, Xiaobo and Tsang, Ivor W},
	journal={IEEE transactions on pattern analysis and machine intelligence},
	volume={44},
	number={11},
	pages={7955--7974},
	year={2021},
	publisher={IEEE}
}

@inproceedings{cca,
	title={Canonical correlation analysis},
	author={Weenink, David},
	booktitle={Proceedings of the Institute of Phonetic Sciences of the University of Amsterdam},
	volume={25},
	pages={81--99},
	year={2003},
	organization={University of Amsterdam Amsterdam}
}

@article{pls,
	title={Partial least squares},
	author={Cha, Jaesung},
	journal={Adv. Methods Mark. Res},
	volume={407},
	pages={52--78},
	year={1994}
}

@article{multiview1,
	title={Asymptotics-aware multi-view subspace clustering},
	author={Xu, Yesong and Chen, Shuo and Li, Jun and Yang, Jian},
	journal={IEEE Transactions on Multimedia},
	volume={27},
	pages={3650--3663},
	year={2025},
	publisher={IEEE}
}

@article{multiview2,
	title={Nonconvex low-rank tensor representation for multi-view subspace clustering with insufficient observed samples},
	author={Ding, Meng and Yang, Jing-Hua and Zhao, Xi-Le and Zhang, Jie and Ng, Michael K},
	journal={IEEE Transactions on Knowledge and Data Engineering},
	year={2025},
	publisher={IEEE}
}

@article{ap1,
	title={Hierarchical text classification with multi-label contrastive learning and KNN},
	author={Zhang, Jun and Li, Yubin and Shen, Fanfan and He, Yueshun and Tan, Hai and He, Yanxiang},
	journal={Neurocomputing},
	volume={577},
	pages={127323},
	year={2024},
	publisher={Elsevier}
}

@inproceedings{ap2,
	title={Learning in imperfect environment: Multi-label classification with long-tailed distribution and partial labels},
	author={Zhang, Wenqiao and Liu, Changshuo and Zeng, Lingze and Ooi, Bengchin and Tang, Siliang and Zhuang, Yueting},
	booktitle={Proceedings of the IEEE/CVF International Conference on Computer Vision},
	pages={1423--1432},
	year={2023}
}

@article{ap3,
	title={A neural network-based multi-label classifier for protein function prediction},
	author={Tahzeeb, Shahab and Hasan, Shehzad},
	journal={Engineering, Technology \& Applied Science Research},
	volume={12},
	number={1},
	pages={7974--7981},
	year={2022}
}

@article{ranksvm2,
	title={Multi-kernel learning for multi-label classification with local Rademacher complexity},
	author={Wang, Zhenxin and Chen, Degang and Che, Xiaoya},
	journal={Information Sciences},
	volume={647},
	pages={119462},
	year={2023},
	publisher={Elsevier}
}

@inproceedings{pml-fp,
	title={Partial multi-label learning},
	author={Xie, Ming-Kun and Huang, Sheng-Jun},
	booktitle={Proceedings of the AAAI conference on artificial intelligence},
	pages={4302--4309},
	year={2018}
}

@inproceedings{lenfn,
	title={Partial Multi-label Learning Based On Near-Far Neighborhood Label Enhancement And Nonlinear Guidance},
	author={Chen, Yu and Wu, Yanan and Han, Na and Fang, Xiaozhao and Chen, Bingzhi and Wen, Jie},
	booktitle={Proceedings of the 32nd ACM International Conference on Multimedia},
	pages={3722--3731},
	year={2024}
}

@article{ml-knn,
	title={ML-KNN: A lazy learning approach to multi-label learning},
	author={Zhang, Min-Ling and Zhou, Zhi-Hua},
	journal={Pattern recognition},
	volume={40},
	number={7},
	pages={2038--2048},
	year={2007},
	publisher={Elsevier}
}

@article{lift,
	title={Lift: Multi-label learning with label-specific features},
	author={Zhang, Min-Ling and Wu, Lei},
	journal={IEEE transactions on pattern analysis and machine intelligence},
	volume={37},
	number={1},
	pages={107--120},
	year={2014},
	publisher={IEEE}
}

@inproceedings{pml-plr,
	title     = {Pseudo-Label Reconstruction for Partial Multi-Label Learning},
	author    = {Chen, Yu and Li, Fang and Han, Na and Li, Guanbin and Gao, Hongbo and Chan, Sixian and Fang, Xiaozhao},
	booktitle = {Proceedings of the Thirty-Fourth International Joint Conference on Artificial Intelligence, {IJCAI-25}},
	pages     = {4896--4904},
	year      = {2025}
}

@article{glc,
	title={Global-local label correlation for partial multi-label learning},
	author={Sun, Lijuan and Feng, Songhe and Liu, Jun and Lyu, Gengyu and Lang, Congyan},
	journal={IEEE Transactions on Multimedia},
	volume={24},
	pages={581--593},
	year={2021},
	publisher={IEEE}
}

@article{particle,
	title={Partial multi-label learning via credible label elicitation},
	author={Zhang, Min-Ling and Fang, Jun-Peng},
	journal={IEEE Transactions on Pattern Analysis and Machine Intelligence},
	volume={43},
	number={10},
	pages={3587--3599},
	year={2020},
	publisher={IEEE}
}

@inproceedings{drama,
	title={Discriminative and Correlative Partial Multi-Label Learning.},
	author={Wang, Haobo and Liu, Weiwei and Zhao, Yang and Zhang, Chen and Hu, Tianlei and Chen, Gang},
	booktitle={IJCAI},
	pages={3691--3697},
	year={2019}
}

@article{penad,
	title={Progressive Enhancement of Label Distributions for Partial Multilabel Learning},
	author={Xu, Ning and Liu, Yun-Peng and Zhang, Yan and Geng, Xin},
	journal={IEEE transactions on neural networks and learning systems},
	volume={34},
	number={8},
	pages={4856--4867},
	year={2023}
}

@article{pamb,
	title={Towards enabling binary decomposition for partial multi-label learning},
	author={Liu, Bing-Qing and Jia, Bin-Bin and Zhang, Min-Ling},
	journal={IEEE transactions on pattern analysis and machine intelligence},
	year={2023},
	publisher={IEEE}
}

@inproceedings{fpml,
	title={Feature-induced partial multi-label learning},
	author={Yu, Guoxian and Chen, Xia and Domeniconi, Carlotta and Wang, Jun and Li, Zhao and Zhang, Zili and Wu, Xindong},
	booktitle={2018 IEEE international conference on data mining (ICDM)},
	pages={1398--1403},
	year={2018},
	organization={IEEE}
}

@inproceedings{pml-lrs,
	title={Partial multi-label learning by low-rank and sparse decomposition},
	author={Sun, Lijuan and Feng, Songhe and Wang, Tao and Lang, Congyan and Jin, Yi},
	booktitle={Proceedings of the AAAI conference on artificial intelligence},
	volume={33},
	pages={5016--5023},
	year={2019}
}

@article{pmfs-lrs,
	title={Partial multi-label feature selection via low-rank and sparse factorization with manifold learning},
	author={Sun, Zhenzhen and Chen, Zexiang and Liu, Jinghua and Chen, Yewang and Yu, Yuanlong},
	journal={Knowledge-Based Systems},
	volume={296},
	pages={111899},
	year={2024},
	publisher={Elsevier}
}

@article{mll-tkde,
	title={Multi-label classification with high-rank and high-order label correlations},
	author={Si, Chongjie and Jia, Yuheng and Wang, Ran and Zhang, Min-Ling and Feng, Yanghe and Qu, Chongxiao},
	journal={IEEE Transactions on Knowledge and Data Engineering},
	volume={36},
	number={8},
	pages={4076--4088},
	year={2023},
	publisher={IEEE}
}

@article{mll-tkde2,
	title={Multi-dimensional classification via decomposed label encoding},
	author={Jia, Bin-Bin and Zhang, Min-Ling},
	journal={IEEE Transactions on Knowledge and Data Engineering},
	volume={35},
	number={2},
	pages={1844--1856},
	year={2021},
	publisher={IEEE}
}

@article{pml-salc,
	title={Partial multi-label learning based on sparse asymmetric label correlations},
	author={Zhao, Peng and Zhao, Shiyi and Zhao, Xuyang and Liu, Huiting and Ji, Xia},
	journal={Knowledge-Based Systems},
	volume={245},
	pages={108601},
	year={2022},
	publisher={Elsevier}
}

@inproceedings{muser,
	title={Partial multi-label learning via multi-subspace representation},
	author={Li, Ziwei and Lyu, Gengyu and Feng, Songhe},
	booktitle={Proceedings of the Twenty-Ninth International Conference on International Joint Conferences on Artificial Intelligence},
	pages={2612--2618},
	year={2021}
}

@article{pml-ni,
	title={Partial multi-label learning with noisy label identification},
	author={Xie, Ming-Kun and Huang, Sheng-Jun},
	journal={IEEE Transactions on Pattern Analysis and Machine Intelligence},
	volume={44},
	number={7},
	pages={3676--3687},
	year={2021},
	publisher={IEEE}
}

@article{pml-dndc,
	title={Dual Noise Elimination And Dynamic Label Correlation Guided Partial Multi-label Learning},
	author={Hu, Yan and Fang, Xiaozhao and Kang, Peipei and Chen, Yonghao and Fang, Yuting and Xie, Shengli},
	journal={IEEE Transactions on Multimedia},
	year={2023},
	publisher={IEEE}
}

@article{pml-nd,
	title={Negative Label and Noise Information Guided Disambiguation for Partial Multi-Label Learning},
	author={Zhong, Jingyu and Shang, Ronghua and Zhao, Feng and Zhang, Weitong and Xu, Songhua},
	journal={IEEE Transactions on Multimedia},
	year={2024},
	publisher={IEEE}
}

@inproceedings{pml-gan,
	title={Adversarial partial multi-label learning with label disambiguation},
	author={Yan, Yan and Guo, Yuhong},
	booktitle={Proceedings of the AAAI Conference on Artificial Intelligence},
	volume={35},
	pages={10568--10576},
	year={2021}
}

@inproceedings{pmlfs,
	title={Partial multi-label feature selection},
	author={Wang, Jing and Li, Peipei and Yu, Kui},
	booktitle={2022 International Joint Conference on Neural Networks (IJCNN)},
	pages={1--9},
	year={2022},
	organization={IEEE}
}

@article{pasad,
	title={Partial multi-label learning via specific label disambiguation},
	author={Li, Feng and Shi, Shengfei and Wang, Hongzhi},
	journal={Knowledge-Based Systems},
	volume={250},
	pages={109093},
	year={2022},
	publisher={Elsevier}
}

@article{pml-fsso,
	title={Partial multi-label feature selection via subspace optimization},
	author={Hao, Pingting and Hu, Liang and Gao, Wanfu},
	journal={Information Sciences},
	volume={648},
	pages={119556},
	year={2023},
	publisher={Elsevier}
}

@article{pard,
	title={Partial multi-label learning with probabilistic graphical disambiguation},
	author={Hang, Jun-Yi and Zhang, Min-Ling},
	journal={Advances in Neural Information Processing Systems},
	volume={36},
	pages={1339--1351},
	year={2023}
}

@article{pml-ed,
	title={PML-ED: A method of partial multi-label learning by using encoder-decoder framework and exploring label correlation},
	author={Wang, Zhenwu and Liu, Fanghan and Han, Mengjie and Tang, Hongjian and Wan, Benting},
	journal={Information Sciences},
	volume={661},
	pages={120165},
	year={2024},
	publisher={Elsevier}
}

@article{pml-bls,
	title={Partial multilabel learning using noise-tolerant broad learning system with label enhancement and dimensionality reduction},
	author={Qian, Wenbin and Tu, Yanqiang and Huang, Jintao and Shu, Wenhao and Cheung, Yiu-Ming},
	journal={IEEE Transactions on Neural Networks and Learning Systems},
	year={2024},
	publisher={IEEE}
}

@article{pml-lc,
	title={Partial multi-label learning with label and classifier correlations},
	author={Wang, Ke and Guan, Yahu and Xie, Yunyu and Jia, Zhaohong and Ye, Hong and Duan, Zhangling and Liang, Dong},
	journal={Information Sciences},
	volume={712},
	pages={122101},
	year={2025},
	publisher={Elsevier}
}

@article{fbd-pml,
	title={Fuzzy bifocal disambiguation for partial multi-label learning},
	author={Fang, Xiaozhao and Hu, Xi and Hu, Yan and Chen, Yonghao and Xie, Shengli and Han, Na},
	journal={Neural Networks},
	volume={185},
	pages={107137},
	year={2025},
	publisher={Elsevier}
}

@article{lcfs-pml,
	title={Integrating label confidence-based feature selection for partial multi-label learning},
	author={Han, Qingqi and Hu, Liang and Gao, Wanfu},
	journal={Pattern Recognition},
	volume={161},
	pages={111281},
	year={2025},
	publisher={Elsevier}
}

@article{pml-ldl,
	title={Partial multi-label feature selection based on label distribution learning},
	author={Lin, Yaojin and Li, Yulin and Lin, Shidong and Guo, Lei and Mao, Yu},
	journal={Pattern Recognition},
	volume={164},
	pages={111523},
	year={2025},
	publisher={Elsevier}
}

@article{cllfs,
	title={Learning accurate label-specific features from partially multilabeled data},
	author={Xu, Tiantian and Xu, Yuanyuan and Yang, Shiyu and Li, Binghao and Zhang, Wenjie},
	journal={IEEE Transactions on Neural Networks and Learning Systems},
	year={2023},
	publisher={IEEE}
}

@article{pmsne,
	title={Partial multi-label feature selection with feature noise},
	author={Wu, You and Li, Peipei and Zou, Yizhang},
	journal={Pattern Recognition},
	volume={162},
	pages={111310},
	year={2025},
	publisher={Elsevier}
}

@article{mll4,
	title={Multi-label classification: An overview},
	author={Tsoumakas, Grigorios and Katakis, Ioannis},
	journal={Data Warehousing and Mining: Concepts, Methodologies, Tools, and Applications},
	pages={64--74},
	year={2008},
	publisher={IGI Global}
}

@inproceedings{mll5,
	title={Multi-label manifold learning},
	author={Hou, Peng and Geng, Xin and Zhang, Min-Ling},
	booktitle={Proceedings of the AAAI conference on artificial intelligence},
	volume={30},
	number={1},
	year={2016}
}

@article{mll6,
	title={Manifold regularized discriminative feature selection for multi-label learning},
	author={Zhang, Jia and Luo, Zhiming and Li, Candong and Zhou, Changen and Li, Shaozi},
	journal={Pattern Recognition},
	volume={95},
	pages={136--150},
	year={2019},
	publisher={Elsevier}
}

@ARTICLE{mll-tkde3,
	author={Zhu, Yue and Kwok, James T. and Zhou, Zhi-Hua},
	journal={IEEE Transactions on Knowledge and Data Engineering}, 
	title={Multi-Label Learning with Global and Local Label Correlation}, 
	year={2018},
	volume={30},
	number={6},
	pages={1081-1094},
}

@article{pll-0,
	title={Partial label learning based on disambiguation correction net with graph representation},
	author={Fan, Jinfu and Yu, Yang and Wang, Zhongjie and Gu, Jinyi},
	journal={IEEE Transactions on Circuits and Systems for Video Technology},
	volume={32},
	number={8},
	pages={4953--4967},
	year={2021},
	publisher={IEEE}
}

@article{pll-tkde,
	title={Learning from noisy labels via dynamic loss thresholding},
	author={Yang, Hao and Jin, You-Zhi and Li, Zi-Yin and Wang, Deng-Bao and Geng, Xin and Zhang, Min-Ling},
	journal={IEEE Transactions on Knowledge and Data Engineering},
	year={2023},
	publisher={IEEE}
}

@inproceedings{pll-2,
	title={Adaptive graph guided disambiguation for partial label learning},
	author={Wang, Deng-Bao and Li, Li and Zhang, Min-Ling},
	booktitle={Proceedings of the 25th ACM SIGKDD international conference on knowledge discovery \& data mining},
	pages={83--91},
	year={2019}
}

@article{pll-3,
	title={Variational Label Enhancement for Instance-Dependent Partial Label Learning},
	author={Xu, Ning and Qiao, Congyu and Zhao, Yuchen and Geng, Xin and Zhang, Min-Ling},
	journal={IEEE Transactions on Pattern Analysis and Machine Intelligence},
	year={2024},
	publisher={IEEE}
}

@article{pll-5,
	title={Disambiguation-free partial label learning},
	author={Zhang, Min-Ling and Yu, Fei and Tang, Cai-Zhi},
	journal={IEEE Transactions on Knowledge and Data Engineering},
	volume={29},
	number={10},
	pages={2155--2167},
	year={2017},
	publisher={IEEE}
}

@article{demvsar2006statistical,
	title={Statistical comparisons of classifiers over multiple data sets},
	author={Dem{\v{s}}ar, Janez},
	journal={The Journal of Machine learning research},
	volume={7},
	pages={1--30},
	year={2006},
	publisher={JMLR. org}
}

@article{briggs2012acoustic,
	title={Acoustic classification of multiple simultaneous bird species: A multi-instance multi-label approach},
	author={Briggs, Forrest and Lakshminarayanan, Balaji and Neal, Lawrence and Fern, Xiaoli Z and Raich, Raviv and Hadley, Sarah JK and Hadley, Adam S and Betts, Matthew G},
	journal={The Journal of the Acoustical Society of America},
	volume={131},
	number={6},
	pages={4640--4650},
	year={2012},
	publisher={AIP Publishing}
}

@inproceedings{snoek2006challenge,
	title={The challenge problem for automated detection of 101 semantic concepts in multimedia},
	author={Snoek, Cees GM and Worring, Marcel and Van Gemert, Jan C and Geusebroek, Jan-Mark and Smeulders, Arnold WM},
	booktitle={Proceedings of the 14th ACM international conference on Multimedia},
	pages={421--430},
	year={2006}
}

@inproceedings{trohidis2008multi,
	title={Multi-label classification of music into emotions.},
	author={Trohidis, Konstantinos and Tsoumakas, Grigorios and Kalliris, George and Vlahavas, Ioannis P and others},
	booktitle={ISMIR},
	volume={8},
	pages={325--330},
	year={2008}
}

@inproceedings{huiskes2008mir,
	title={The mir flickr retrieval evaluation},
	author={Huiskes, Mark J and Lew, Michael S},
	booktitle={Proceedings of the 1st ACM international conference on Multimedia information retrieval},
	pages={39--43},
	year={2008}
}

@article{zhang2013review,
	title={A review on multi-label learning algorithms},
	author={Zhang, Min-Ling and Zhou, Zhi-Hua},
	journal={IEEE transactions on knowledge and data engineering},
	volume={26},
	number={8},
	pages={1819--1837},
	year={2013},
	publisher={IEEE}
}
	\newpage
	\appendices
	\section{Optimization}
	
	\subsection{Overall Objective Function}	
	\begin{equation}\label{eqx}
		\begin{aligned}
			&\begin{aligned}
				\min_{\substack{\mathbf{P}_1,\mathbf{P}_2,\mathbf{Q},\\ 
						\mathbf{R},\mathbf{V},\mathbf{D}}}~
				&{\|\mathbf{X}\mathbf{P}_1- \mathbf{R}\mathbf{P}_2\|_F^2+\|\mathbf{Y}- \mathbf{R}\mathbf{Q}^\top\|_{F}^{2}+\lambda\|\mathbf{R}\|_*}\\
				&+\alpha\sum_{i=1}^{n}\sum_{j=1}^{n}
				s_{ij}\|\mathbf{P}_1^\top \mathbf{x}_{i}-\mathbf{P}_2^\top \mathbf{r}_{j}\|_{2}^{2}\\
				&+\beta\sum_{i=1}^{n}\left\|\mathbf{x}_id_{ii}-\sum_{j=1}^{c}r_{ij}\mathbf{v}_j\right\|_{2}^{2}+\gamma\|\mathbf{P}_2\mathbf{P}_1^\top \|_{F}^{2},\\
			\end{aligned}\\
			&\text{s.t.} ~~\mathbf{Q}^{\top}\mathbf{Q}=\mathbf{I}_{c}, ~\mathbf{P}_2^{\top}\mathbf{P}_2=\mathbf{I}_{m},~0\leq r_{ij}\leq y_{ij},\forall i,j,\\
		\end{aligned}
	\end{equation}
	\subsection{Optimization}	
	
	We employ an alternating optimization approach to solve Eq. \eqref{eqx}. In this method, each variable is optimized while keeping the other variables fixed, and all variables are updated iteratively in turns. This process continues until the loss function converges.
	
	\subsubsection{Update \texorpdfstring{$\mathbf{P}_1$ and $\mathbf{P}_2$}{P}}
	Removing the terms that are irrelevant to $\mathbf{P}_1$ and $\mathbf{P}_2$, the suboptimization problem derived from Eq. \eqref{eqx} is simplified as:
	\begin{equation}\label{p1p2}
		\begin{aligned}
			&\min_{\mathbf{P}_1,\mathbf{P}_2}~~
			\|\mathbf{X}\mathbf{P}_1-\mathbf{R}\mathbf{P}_2\|_F^2
			+\alpha\sum_{i=1}^{n}\sum_{j=1}^{n}s_{ij}
			\|\mathbf{P}_1^\top \mathbf{x}_{i}-\mathbf{P}_2^\top \mathbf{r}_{j}\|_{2}^{2}\\
			&+\gamma\|\mathbf{P}_2\mathbf{P}_1^\top \|_{F}^{2},\quad\text{s.t.}\quad\mathbf{P}_2^\top\mathbf{P}_2=\mathbf{I}_{m}.
		\end{aligned}
	\end{equation}
	
	The local alignment term in Eq. \eqref{p1p2} can be converted to trace form:
	\begin{equation}\label{lma}
		\begin{aligned}		
			&\sum_{i=1}^{n}\sum_{j=1}^{n}
			s_{ij}\|\mathbf{P}_1^\top \mathbf{x}_{i}-\mathbf{P}_2^\top \mathbf{r}_{j}\|_{2}^{2}\\
			&=Tr(\mathbf{P}_1^\top \mathbf{X}^\top \mathbf{D}^s\mathbf{X}\mathbf{P}_1)
			+Tr(\mathbf{P}_2^\top \mathbf{R}^\top\mathbf{D}^s\mathbf{R}\mathbf{P}_2)\\
			&\quad-2Tr(\mathbf{P}_1^\top \mathbf{X}^\top\mathbf{S}\mathbf{R}\mathbf{P}_2),
		\end{aligned}
	\end{equation}
	where $Tr(\cdot)$ denotes the trace norm, and $\mathbf{D}^s$ is a diagonal degree matrix with $\mathbf{D}^s_{ii}=\sum_{j=1}^{n}\mathbf{S}_{ij}$. Since $\mathbf{S}$ is symmetric, the row and column degree matrices are identical.
	
	\textbf{Update $\mathbf{P}_1$:} Taking the derivative of Eq. \eqref{p1p2} with respect to $\mathbf{P}_1$ and setting it to zero yields:
	\begin{equation}\label{p1}
		\underbrace{\mathbf{X}^\top(\mathbf{X} +\alpha\mathbf{D^sX}) }_{A_1}\mathbf{P}_1+\mathbf{P}_1\underbrace{\beta \mathbf{P}_2^\top \mathbf{P}_2}_{B_1}
		=\underbrace{\mathbf{X}^\top(\mathbf{RP}_2+\alpha\mathbf{SRP}_2)}_{C_1}
	\end{equation}
	Rearranging Eq. \eqref{p1}, we obtain the \textit{Sylvester} equation for $\mathbf{P}_1$:
	\begin{equation}\label{p11}
		\mathbf{A}_1\mathbf{P}_1 + \mathbf{P}_1\mathbf{B}_1 = \mathbf{C}_1,
	\end{equation}
	Then the final solution of $\mathbf{P}_1$ can be achieved by  the \textit{lyap} or \textit{sylvester} function in MATLAB:
	\begin{equation}\label{p111}
		\mathbf{P}_1^{t+1} = sylvester(\mathbf{A}_1,\mathbf{B}_1,\mathbf{C}_1).
	\end{equation}
	
	\textbf{Update $\mathbf{P}_2$:} Similarly, taking the derivative with respect to $\mathbf{P}_2$ and setting it to zero:
	\begin{equation}\label{p2}
		\underbrace{\mathbf{R}^\top(\mathbf{R} +\alpha\mathbf{D^sR}) }_{A_2}\mathbf{P}_2+\mathbf{P}_2\underbrace{\beta \mathbf{P}_1^\top \mathbf{P}_1}_{B_2}
		=\underbrace{\mathbf{R}^\top(\mathbf{XP}_1+\alpha\mathbf{SXP}_1)}_{C_2}
	\end{equation}
	This leads to the Sylvester equation for $\mathbf{P}_2$:
	\begin{equation}\label{p22}
		\hat{\mathbf{P}}_2 = sylvester(\mathbf{A}_1,\mathbf{B}_1,\mathbf{C}_1).
	\end{equation}
	
	To satisfy the orthogonality constraint $\mathbf{P}_2^\top\mathbf{P}_2=\mathbf{I}_{m}$, we project $\hat{\mathbf{P}}_2$ onto the orthogonal manifold via singular value decomposition (SVD). Let $\hat{\mathbf{P}}_2 = \mathbf{U}_P\mathbf{\Sigma}_P\mathbf{V}_P^\top$ be the SVD of $\hat{\mathbf{P}}_2$, then the optimal orthogonal solution is:
	\begin{equation}\label{p222}
		\mathbf{P}_2^{t+1}=\mathbf{U}_p\mathbf{V}_P^\top .
	\end{equation}
	\subsubsection{Update \texorpdfstring{$\mathbf{Q}$}{P}}
	Removing the items that are irrelevant to $\mathbf{Q}$, and fix $\mathbf{R}$.
	Then, the suboptimization problem derived in Eq. \eqref{eqx} is simplified as:
	\begin{equation}\label{q1}
		\min_{\substack{\mathbf{Q}}}~\|\mathbf{Y}-\mathbf{R}\mathbf{Q}^\top \|_{F}^{2}, \quad\text{s.t.}\quad \mathbf{Q}^{\top}\mathbf{Q}=\mathbf{I_{c}}.
	\end{equation}
	Since $\|\mathbf{Y}-\mathbf{R}\mathbf{Q}^\top\|_{F}^{2}=Tr(\mathbf{Y}^\top \mathbf{Y}+\mathbf{R}\mathbf{Q}^\top\mathbf{QR}^\top-2\mathbf{Q}\mathbf{R}^\top \mathbf{Y})$ and $\mathbf{Q}^\top\mathbf{Q} =\mathbf{I}_{c}$, for $\mathbf{Y}^\top \mathbf{Y}$ and $\mathbf{R}\mathbf{Q}^\top\mathbf{QR}^\top$, these values are constant. Alternatively, the optimization problems in Eq. \eqref{q1} can be stated as maximizing the following trace functions:
	\begin{equation}\label{q2}
		\begin{aligned}
			&\max_{\substack{\mathbf{Q}}}~Tr(\mathbf{Q}\mathbf{R}^\top \mathbf{Y}),
			&\text{s.t.} ~~\mathbf{Q}^{\top}\mathbf{Q}=\mathbf{I}_{c}.\\
		\end{aligned}
	\end{equation}
	Similarly, the optimal solutions of $\mathbf{Q}$ can be approximated by respectively computing the SVD of $\mathbf{R}^\top\mathbf{Y} \approx\mathbf{U}_Q\mathbf{\Sigma}_Q\mathbf{V}_Q^\top$, the final solution of $\mathbf{Q}$ can be achieved by:
	\begin{equation}\label{q3}
		\mathbf{Q}=\mathbf{V}_Q\mathbf{U}_Q^\top.
	\end{equation}
	\subsubsection{Update \texorpdfstring{$\mathbf{V}$}{P}}
	Removing the items that are irrelevant to $\mathbf{V}$, and fix $\mathbf{R}$ and $\mathbf{D}$.
	Then, the suboptimization problem derived in Eq. \eqref{eqx} is simplified as:
	\begin{equation}\label{v1}
		\begin{aligned}
			\min_{\substack{\mathbf{V}}}~\sum_{i=1}^{n}\left\|\mathbf{x}_id_{ii}-\sum_{j=1}^{c}r_{ij}\mathbf{v}_j\right\|_{2}^{2}=
			\|\mathbf{X}^\top\mathbf{D} -\mathbf{VR}^\top\|_{F}^{2}.\\
		\end{aligned}
	\end{equation}
	After differentiating Eq. \eqref{v1}, set the derivative to 0, the updated formula for $\mathbf{V}$ can be obtained as follows:
	\begin{equation}\label{eq24}
		\mathbf{V}^{t+1}=\mathbf{X}_1^\top \mathbf{DR}(\mathbf{R}^\top\mathbf{R})^{-1}.\\
	\end{equation}
	or
	\begin{equation}\label{eq:update_prototypes}
		\mathbf{v}_j^{t+1} = \frac{\sum_{i=1}^{n} r_{ij} \mathbf{x}_i}{\sum_{i=1}^{n}  r_{ij}}, \quad j = 1, 2, \ldots, q.
	\end{equation}
	\subsubsection{Update \texorpdfstring{$\mathbf{R}$ and $\mathbf{D}$}{R and D}}
	Removing the terms that are irrelevant to $\mathbf{R}$, and fixing $\mathbf{P}_1$, $\mathbf{P}_2$, $\mathbf{Q}$ and $\mathbf{V}$. First, we update the scalar weight matrix $\mathbf{D}$ using the result of the current iteration round $\mathbf{R}^{t}$:
	\begin{equation}\label{update_d}
		d_{ii}^{t+1}=\sum_{j=1}^{c}r_{ij}^{t}.
	\end{equation}
	
	Then, the suboptimization problem for $\mathbf{R}$ derived from Eq. \eqref{eqx} is:
	\begin{equation}\label{r}
		\begin{aligned}
			&\min_{\mathbf{R}}~
			\|\mathbf{X}\mathbf{P}_1- \mathbf{R}\mathbf{P}_2\|_F^2+\|\mathbf{Y}- \mathbf{R}\mathbf{Q}^\top\|_{F}^{2}+\lambda\|\mathbf{R}\|_*\\
			&+\alpha\sum_{i=1}^{n}\sum_{j=1}^{n}
			s_{ij}\|\mathbf{P}_1^\top \mathbf{x}_{i}-\mathbf{P}_2^\top \mathbf{r}_{j}\|_{2}^{2}
			+\beta\|\mathbf{X}^\top\mathbf{D} -\mathbf{VR}^\top\|_{F}^{2}\\
			&\text{s.t.} \quad~0\leq r_{ij}\leq y_{ij},~\forall i,j.
		\end{aligned}
	\end{equation}
	
	Since the objective function contains both differentiable and non-differentiable terms, we employ the proximal gradient descent (PGD) method. We split the objective function $\mathcal{L}(\mathbf{R})$ into a differentiable part $\mathcal{J}(\mathbf{R})$ and a non-differentiable part $\mathcal{K}(\mathbf{R})$:
	\begin{equation}\label{eq26}
		\mathcal{L}(\mathbf{R})=\mathcal{J}(\mathbf{R})+\mathcal{K}(\mathbf{R}),
	\end{equation}
	where
	\begin{equation}\label{eq27}
		\mathcal{K}(\mathbf{R}) =\lambda\|\mathbf{R}\|_*.
	\end{equation}
	
	The gradient of $\mathcal{J}(\mathbf{R})$ is:
	\begin{equation}\label{gradient_R}
		\begin{aligned}
			\nabla \mathcal{J}(\mathbf{R}) =& 
			2\mathbf{R}\mathbf{P}_2\mathbf{P}_2^\top - 2\mathbf{X}\mathbf{P}_1\mathbf{P}_2^\top
			+ 2\mathbf{R}\mathbf{Q}^\top\mathbf{Q} - 2\mathbf{Y}\mathbf{Q}\\
			&+ 2\alpha\mathbf{D}^s\mathbf{R}\mathbf{P}_2\mathbf{P}_2^\top - 2\alpha\mathbf{S}\mathbf{X}\mathbf{P}_1\mathbf{P}_2^\top\\
			&+ 2\beta\mathbf{R}\mathbf{V}^\top\mathbf{V} - 2\beta\mathbf{D}\mathbf{X}\mathbf{V}.
		\end{aligned}
	\end{equation}
	
	The proximal operator at iteration $t$ is:
	\begin{equation}\label{eq28}
		\mathrm{prox}_{\mathcal{K}}(\mathbf{Z})=\arg\min_{\mathbf{R}} \frac{1}{2}\|\mathbf{R}-\mathbf{Z}\|_{F}^{2}+\frac{\lambda}{L}\|\mathbf{R}\|_{*},
	\end{equation}
	where $\mathbf{Z} = \mathbf{R}^t - \frac{1}{L} \nabla \mathcal{J}(\mathbf{R}^t)$, and $L$ is the Lipschitz constant of $\nabla \mathcal{J}(\mathbf{R})$.
	Eq. \eqref{eq28} can be solved using the singular value thresholding (SVT) operator. By performing SVD on $\mathbf{Z}=\mathbf{U}_Z\mathbf{\Sigma}_Z\mathbf{V}_Z^\top$, we obtain:
	\begin{equation}\label{svt}
		{
			\hat{\mathbf{R}}=\mathbf{U}_Z~\mathcal{S}_{\lambda/{L}}(\mathbf{\Sigma}_Z)~\mathbf{V}_Z^\top,	}
	\end{equation}
	{where $\mathcal{S}_{\lambda/{L}}(\mathbf{\Sigma}_Z)$ is the soft-thresholded singular value matrix with diagonal entries:}
	\begin{equation}\label{threshold}
		{\mathcal{S}_{\lambda/{L}}(\mathbf{\Sigma}_Z)_{ii}=\max\left((\mathbf{\Sigma}_Z)_{ii}-\frac{\lambda}{L},0\right).}
	\end{equation}
	
	To compute the Lipschitz constant $L$, we analyze the gradient difference. {Note that the $\mathbf{R}$-dependent terms in $\nabla\mathcal{J}(\mathbf{R})$ involve right-multiplication by $c\times c$ matrices ($\mathbf{P}_2\mathbf{P}_2^\top$, $\mathbf{Q}^\top\mathbf{Q}$, $\mathbf{V}^\top\mathbf{V}$) and left-multiplication by the $n\times n$ matrix $\mathbf{D}^s$. Thus:}
	\begin{equation}\label{lipschitz2}
		\begin{aligned}
			&{\|\nabla\mathcal{J}(\mathbf{R}_1)-\nabla\mathcal{J}(\mathbf{R}_2)\|_F}\\
			&{\leq 2\|\mathbf{P}_2\mathbf{P}_2^\top + \mathbf{Q}^\top\mathbf{Q} + \beta\mathbf{V}^\top\mathbf{V}\|_2\cdot\|\Delta\mathbf{R}\|_F}\\
			&{\quad+ 2\alpha\|\mathbf{D}^s\|_2\cdot\|\mathbf{P}_2\mathbf{P}_2^\top\|_2\cdot\|\Delta\mathbf{R}\|_F.}
		\end{aligned}
	\end{equation}
	{Then the Lipschitz constant $L$ is:}
	\begin{equation}\label{L_value}
		{L=2\left(\|\mathbf{P}_2\mathbf{P}_2^\top + \mathbf{Q}^\top\mathbf{Q} + \beta\mathbf{V}^\top\mathbf{V}\|_2 + \alpha\cdot\max_i D^s_{ii}\right).}
	\end{equation}
	
	Finally, to enforce the box constraint $0\leq r_{ij}\leq y_{ij}$, we decompose it into two non-negative constraints: $\mathbf{R}\geq 0$ and $\mathbf{Y}-\mathbf{R}\geq 0$. We apply the Lagrange multiplier method with the following formulation:
	\begin{equation}\label{eq32}
		\min_{\mathbf{R}}\|\mathbf{R}-\hat{\mathbf{R}}\|_F^2-Tr(\mathbf{\Phi}^\top\mathbf{R})-Tr(\mathbf{\Psi}^\top(\mathbf{Y}-\mathbf{R})),
	\end{equation}
	where $\mathbf{\Phi}$ and $\mathbf{\Psi}$ represent the Lagrange multiplier. Taking the derivative of Eq. \eqref{eq32} and setting the derivative to zero. Then
	multiplying both sides of {Eq. \eqref{eq34}} element-wise by \( \mathbf{R} \), the equation still holds, and the updated equation becomes:
	\begin{equation}\label{eq34}
		\mathbf{R}\odot(\mathbf{R}-\hat{\mathbf{R}}-\mathbf{\Phi}+\mathbf{\Psi})=0.
	\end{equation}
	Based on condition of Karush-Kuhn-Tucker (KKT), it can
	be given that: $\mathbf{\Phi}_{ij}\mathbf{R}_{ij}=0$ and $\mathbf{\Psi}_{ij}\mathbf{(R-Y)}_{ij}=0$
	the update rules for $\mathbf{R}$ can be obtained as:
	\begin{equation}\label{eq35}
		\mathbf{R}_{ij}^{t+1}=\mathbf{R}_{ij}^{t}\frac{\hat{\mathbf{R}}_{ij}}{\mathbf{R}_{ij}+\epsilon}-\frac{\mathbf{\Psi}_{ij}\mathbf{Y}_{ij}}{\mathbf{R}_{ij}+\epsilon},
	\end{equation}
	where $\epsilon$ is a small positive constant (e.g., $10^{-10}$) to prevent numerical instability. The update for \( \mathbf{\Psi} \) is based on the complementary condition: $\mathbf{\Psi}_{ij}=\max(0,\mathbf{R}_{ij}-\mathbf{Y}_{ij})$.
	\begin{algorithm}[htbp]
		\caption{PML-MA: Feature-Label Modal Alignment for Partial Multi-Label Learning}
		\label{alg:pml-ma}
		\renewcommand{\algorithmicrequire}{\textbf{Input:}}
		\renewcommand{\algorithmicensure}{\textbf{Output:}}
		\begin{algorithmic}[1]
			\REQUIRE 
			Feature matrix $\mathbf{X} \in \mathbb{R}^{n \times d}$, 
			Candidate label matrix $\mathbf{Y} \in \{0,1\}^{n \times c}$,
			Parameters $\lambda, \alpha, \beta, \gamma$,
			Number of neighbors $k$,
			Common subspace dimension $m$,
			Maximum iterations $T_{max}$, tolerance $\epsilon$.
			\ENSURE 
			Classifier $\mathbf{W} = \mathbf{P}_1\mathbf{P}_2^\top$.
			\STATE \textbf{Initialization:}
			\STATE Compute similarity matrix $\mathbf{S}$ using Eq. (5)(In the main text) and degree matrix $\mathbf{D}^s$;
			\STATE Initialize $\mathbf{P}_1 \in \mathbb{R}^{d \times m}$, $\mathbf{P}_2 \in \mathbb{R}^{c \times m}$ randomly with orthonormal columns;
			\STATE Initialize $\mathbf{Q} = \mathbf{I}_c$, $\mathbf{R} = \mathbf{Y}$, $\mathbf{V} \in \mathbb{R}^{d \times c}$ randomly;
			\STATE Initialize scalar weight matrix $\mathbf{D}$ with $d_{ii} = \sum_{j=1}^{c} r_{ij}$;
			\STATE Set $t = 0$, $\mathcal{L}^{old} = \infty$;
			\WHILE{$t < T_{max}$ and not converged}

			\STATE // \textbf{Update $\mathbf{P}_2$} {(Eqs.~\eqref{p2}--\eqref{p222})}
			\STATE Compute $\mathbf{A}_2 = \mathbf{R}^\top(\mathbf{I}_n + \alpha\mathbf{D}^s)\mathbf{R}$;
			\STATE Compute $\mathbf{B}_2 = \gamma\mathbf{P}_1^\top\mathbf{P}_1$;
			\STATE Compute $\mathbf{C}_2 = \mathbf{R}^\top(\mathbf{I}_n + \alpha\mathbf{S})\mathbf{X}\mathbf{P}_1$;
			\STATE Solve Sylvester equation: \\$\hat{\mathbf{P}}_2 = \text{sylvester}(\mathbf{A}_2, \mathbf{B}_2, \mathbf{C}_2)$;
			\STATE Perform SVD: $\hat{\mathbf{P}}_2 = \mathbf{U}_P\mathbf{\Sigma}_P\mathbf{V}_P^\top$;
			\STATE Project to orthogonal manifold: $\mathbf{P}_2^{t+1} = \mathbf{U}_P\mathbf{V}_P^\top$;
			
			\STATE // \textbf{Update $\mathbf{P}_1$} {(Eqs.~\eqref{p1}--\eqref{p111})}
			\STATE Compute $\mathbf{A}_1 = \mathbf{X}^\top(\mathbf{I}_n + \alpha\mathbf{D}^s)\mathbf{X}$;
			\STATE Compute $\mathbf{B}_1 = \gamma\mathbf{P}_2^\top\mathbf{P}_2$;
			\STATE Compute $\mathbf{C}_1 = \mathbf{X}^\top(\mathbf{I}_n + \alpha\mathbf{S})\mathbf{R}\mathbf{P}_2$;
			\STATE Solve Sylvester equation:\\ ${\mathbf{P}}_1^{t+1} = \text{sylvester}(\mathbf{A}_1, \mathbf{B}_1, \mathbf{C}_1)$;
			
			\STATE // \textbf{Update $\mathbf{Q}$} {(Eqs.~\eqref{q1}--\eqref{q3})}
			\STATE Compute $\mathbf{Q} = \mathbf{R}^\top\mathbf{Y}$;
			\STATE Perform SVD: $\mathbf{Q} = \mathbf{U}_Q\mathbf{\Sigma}_Q\mathbf{V}_Q^\top$;
			\STATE Project to orthogonal manifold: $\mathbf{Q}^{t+1} = \mathbf{V}_Q\mathbf{U}_Q^\top$;
			
			\STATE // \textbf{Update $\mathbf{V}$} {(Eq.~\eqref{eq24})}
			\STATE Compute $\mathbf{V}^{t+1} = (\mathbf{R}^\top\mathbf{R})^{-1}\mathbf{R}^\top\mathbf{D}\mathbf{X}$;
			
			\STATE // \textbf{Update $\mathbf{R}$} {(Eqs.~\eqref{gradient_R}--\eqref{eq35})}
			\STATE Update scalar weight matrix: $d_{ii} = \sum_{j=1}^{c} r_{ij}$;
			\STATE Compute gradient $\nabla \mathcal{J}(\mathbf{R})$ using Eq. {\eqref{gradient_R}};
			\STATE Compute Lipschitz constant $L$ using Eq. {\eqref{L_value}};
			\STATE Compute $\mathbf{Z} = \mathbf{R} - \frac{1}{L}\nabla \mathcal{J}(\mathbf{R})$;
			\STATE Perform SVD: $\mathbf{Z} = \mathbf{U}_Z\mathbf{\Sigma}_Z\mathbf{V}_Z^\top$;
			\STATE Apply soft thresholding: $\mathcal{S}_{\lambda/{L}}(\mathbf{\Sigma}_Z)_{ii} = \max(\mathbf{\Sigma}_Z - \frac{\lambda}{L}, 0)$;
			\STATE Reconstruct: $\hat{\mathbf{R}} = \mathbf{U}_Z~\mathcal{S}_{\lambda/{L}}(\mathbf{\Sigma}_Z)\mathbf{V}_Z^\top$;
			\STATE Apply box constraints {(Eq.~\eqref{eq35})}:\\ $	\mathbf{R}_{ij}^{t+1}=\mathbf{R}_{ij}^{t}\frac{\hat{\mathbf{R}}_{ij}}{\mathbf{R}_{ij}+\epsilon}-\frac{\mathbf{\Psi}_{ij}\mathbf{Y}_{ij}}{\mathbf{R}_{ij}+\epsilon}$;
			
			\STATE // \textbf{Update $\mathbf{D}$}
			\STATE Update scalar weight matrix: $d_{ii} = \sum_{j=1}^{c} r_{ij}$, $\forall i$;
			
			\STATE // \textbf{Check Convergence}
			\STATE Compute objective value $\mathcal{L}^{new}$ using Eq. (14);
			\IF{$|\mathcal{L}^{new} - \mathcal{L}^{old}|/|\mathcal{L}^{old}| < \epsilon$}
			\STATE \textbf{break};
			\ENDIF
			\STATE $\mathcal{L}^{old} = \mathcal{L}^{new}$, $t = t + 1$;
			\ENDWHILE
			\STATE Construct final classifier: $\mathbf{W} = \mathbf{P}_2\mathbf{P}_1^\top$;
			\RETURN $\mathbf{W}$;
		\end{algorithmic}
	\end{algorithm}
	\section{Theoretical Justification of Theorem 1}
	\subsection{Complexity Analysis}	
	In this section, we analyze the per-iteration computational complexity of the proposed PML-MA method, which mainly arises from solving the alternating optimization problems for six variables: $\mathbf{P}_1$, $\mathbf{P}_2$, $\mathbf{Q}$, $\mathbf{R}$, $\mathbf{V}$, and $\mathbf{D}$. The computational costs at each iteration are as follows: updating $\mathbf{P}_1$ requires computing $\mathbf{X}^\top(\mathbf{I}_n + \alpha\mathbf{D}^s)\mathbf{X}$ and solving a Sylvester equation of size $d\times m$, yielding $\mathcal{O}(n^2d + nd^2 + d^2m + ncm)$; updating $\mathbf{P}_2$ involves similar operations with additional SVD projection, costing $\mathcal{O}(n^2c + nc^2 + cm^2 + ndm)$; updating $\mathbf{Q}$ requires matrix multiplication and SVD for orthogonality constraint, costing $\mathcal{O}(nc^2 + c^3)$; updating $\mathbf{V}$ involves computing class prototypes through matrix operations, costing $\mathcal{O}(nc^2 + ndc + c^3)$; updating $\mathbf{R}$ via proximal gradient descent with SVD-based soft thresholding costs $\mathcal{O}(nc^2 + ndc + ncm + c^3)$; and updating $\mathbf{D}$ by row summation costs $\mathcal{O}(nc)$. Since $c, m \ll n, d$ in typical multi-label scenarios, the dominant terms per iteration are $\mathcal{O}(n^2d + nd^2 + n^2c + nc^2 + ndc)$. Therefore, for $T$ iterations until convergence, the total computational complexity of PML-MA is $\mathcal{O}(T(n^2d + nd^2 + n^2c + nc^2 + ndc))$, where $T$ is typically small (e.g., $T \leq 20$ as demonstrated in {Fig. 5 in the main paper}), ensuring practical efficiency for real-world applications.
	\subsection{Generalization Bound}	
	This section theoretically analyzes the Rademacher complexity for PML-MA, a common tool for deriving data-dependent risk bounds. 
	\begin{theorem}\label{thm1}  
		Assume that \( \mathcal{G} \) is a class of functions mapping \( \mathbf{X} \) to \( [0,1] \), and \( \mathbf{Z} = \{x_1, x_2, \dots, x_n\} \) represents a fixed sample of size \( n \). The empirical Rademacher complexity of \( \mathcal{G} \) for sample \( \mathbf{Z} \) is given by:  
		\begin{equation}  
			\widehat{\mathcal{R}}_Z(\mathcal{G}) = \mathbb{E}_{\sigma} \left[ \sup_{g \in \mathcal{G}} \frac{1}{n} \sum_{i=1}^n \sigma_i g(x_i) \right],  
		\end{equation}  
		where \( \sigma_i \) is a random variable from \( \{ -1, +1 \} \), and \( \sigma = \{ \sigma_1, \dots, \sigma_n \} \) represents the Rademacher variables.  
	\end{theorem}  
	\begin{lemma}\label{lemma1}  
		Let \( \mathbf{W} = \mathbf{P}_1 \mathbf{P}_2^\top \) be the linear classifier, and let \( \mathcal{H} = \mathbf{W} \times \mathbf{Z} \) be the family of functions for PML-MA with linear functions \( \mathbf{W} \in \mathcal{H} \). For the loss function \( \ell \), the Rademacher complexity of the proposed algorithm is bounded as:  
		\begin{equation}  
			\widehat{\mathcal{R}}_Z(\ell \circ \mathcal{H}) \leq \frac{\sqrt{2}(2c)}{n} \mathbb{E}_{\sigma} \left[ \sup_{\mathbf{W} \in \mathcal{H}} \langle \mathcal{H}^\top, \widehat{\mathbf{X}} \rangle \right],  
		\end{equation}  
		where \( \widehat{\mathbf{X}} \in \mathbb{R}^{d \times c} \) is the weight summation matrix, and \( \widehat{\mathbf{X}}_j = \sum_{i=1}^n \sigma_{ij} \mathbf{x}_i \) represents the weight summation of feature vectors for the \( j \)-th label. Arranging \( \widehat{\mathbf{X}}_j \) for each of the \( c \) labels in columns forms the matrix \( \widehat{\mathbf{X}} \).  
	\end{lemma}  
	\begin{IEEEproof}\label{proof1}  
		By Definition \ref{thm1}, we can reformulate the Rademacher complexity for \( \mathcal{H} \) and \( \ell \) as follows:  
		\begin{equation}  
			\widehat{\mathcal{R}}_Z(\ell \circ \mathcal{H}) = \frac{1}{n} \mathbb{E}_{\sigma} \left[ \sup_{h \in \mathcal{H}} \sum_{i=1}^n \sigma_i \ell(h(x_i), y_i) \right],  
		\end{equation}  
		where \( h \in \mathcal{H} \) is the classifier. The square loss is \( 2c \)-Lipschitz for PML-MA, and by applying the contraction inequality for Rademacher complexity, we get:  
		\begin{equation}  
			\widehat{\mathcal{R}}_Z(\ell \circ \mathcal{H}) \leq \frac{\sqrt{2}(2c)}{n} \mathbb{E}_{\sigma} \left[ \sup_{h \in \mathcal{H}} \sum_{i=1}^n \sum_{j=1}^q \sigma_{ij} h_j(x_i) \right],  
		\end{equation}  
		where \( \sigma_{ij} \) is a doubly indexed Rademacher sequence and \( h_j(x_i) \) denotes the \( j \)-th component of \( h(x_i) \). Thus, the Rademacher complexity of PML-MA is bounded as:  
		\begin{equation}  
			\widehat{\mathcal{R}}_Z(\ell \circ \mathcal{H}) \leq \frac{\sqrt{2}(2c)}{n} \mathbb{E}_{\sigma} \left[ \sup_{h \in \mathcal{H}} \sum_{i=1}^n \sum_{j=1}^c \sigma_{ij} (x_i \cdot w_j) \right],  
		\end{equation}  
		where \( w_j \) is the \( j \)-th column of \( \mathbf{W} \), corresponding to the \( j \)-th label classifier. Thus, the weight summation matrix \( \widehat{\mathbf{X}} \) is constructed as:  
		\begin{equation}  
			\widehat{\mathcal{R}}_Z(\ell \circ \mathcal{H}) \leq \frac{\sqrt{2}(2c)}{n} \mathbb{E}_{\sigma} \left[ \sup_{\mathbf{W} \in \mathcal{H}} \langle \mathcal{H}^\top, \widehat{\mathbf{X}} \rangle \right],  
		\end{equation}  
		completing the proof of Lemma \ref{lemma1}.  
	\end{IEEEproof}  
	\subsection{Consistency Analysis}
	To demonstrate the effectiveness of the proposed model in partial multi-label learning (PML) scenarios, we analyze and prove that under specific conditions, the solution \( \mathbf{R}^t \) of the optimization problem approximates the ground-truth label matrix \( \mathbf{T} \) and aligns consistently with the feature matrix \( \mathbf{X} \). Below is the detailed analysis. 
	\begin{theorem}\label{thm2}
		Assume that the ground-truth label matrix \( \mathbf{T} \) is approximately low-rank, i.e., \( \text{rank}(\mathbf{T}) \leq r \), where \( r \) is a moderate value and \( r < c \) (the number of classes). The candidate label matrix \( \mathbf{Y} \) is a noisy version of \( \mathbf{T} \), and it can be expressed as:
		\begin{equation}
			\mathbf{Y} = \mathbf{T} + \mathbf{E},
		\end{equation}
		where \( \mathbf{E} \) is a sparse noise matrix with \( \|\mathbf{E}\|_F \ll \|\mathbf{T}\|_F \).
	\end{theorem}
	\begin{lemma}\label{lemma2}
		Under the assumptions of Definition \ref{thm2} , the solution \( \mathbf{R}^t \) to the optimization problem
		\begin{equation}
			\min_{\mathbf{Q}, \mathbf{R}} \|\mathbf{Y} -  \mathbf{R}\mathbf{Q}^\top\|_F^2 + \lambda \|\mathbf{R}\|_*,~ {s.t.}~\mathbf{Q}^{\top}\mathbf{Q}=\mathbf{I}_{c},\mathbf{0}_{n \times c}\leq \mathbf{R}\leq \mathbf{Y}.
		\end{equation}
		approximates the ground-truth label matrix \( \mathbf{T} \), and satisfies:
		\begin{equation}
			\mathbf{R}^t \approx \mathbf{T} \quad \text{and} \quad \text{rank}(\mathbf{R}^t) \leq r.
		\end{equation}
	\end{lemma}
	\begin{IEEEproof}
		First, since $\mathbf{T}$ is low-rank with $\text{rank}(\mathbf{T}) \leq r$, it can be decomposed as:
		\begin{equation}
			\mathbf{T} = \mathbf{U}_T\mathbf{\Sigma}_T\mathbf{V}_T^\top,
		\end{equation}
		where $\mathbf{U}_T \in \mathbb{R}^{n \times r}$, $\mathbf{\Sigma}_T \in \mathbb{R}^{r \times r}$, and $\mathbf{V}_T \in \mathbb{R}^{r \times c}$.
		
		Given that $\mathbf{Y} = \mathbf{T} + \mathbf{E}$ and $\|\mathbf{E}\|_F \ll \|\mathbf{T}\|_F$, we can write:
		\begin{equation}
			\|\mathbf{Y} - \mathbf{R}\mathbf{Q}^\top\|_F^2 = \|\mathbf{T} + \mathbf{E} - \mathbf{R}\mathbf{Q}^\top\|_F^2.
		\end{equation}
		
		The nuclear norm regularization term $\|\mathbf{R}\|_*$ promotes low-rank solutions. Since $\mathbf{T}$ is already low-rank with $\text{rank}(\mathbf{T}) \leq r$, and $\|\mathbf{E}\|_F$ is small, the optimal solution $\mathbf{R}^t$ will also be approximately low-rank with $\text{rank}(\mathbf{R}^t) \leq r$.
		
		Moreover, the constraint $\mathbf{0}_{c \times n} \leq \mathbf{R} \leq \mathbf{Y}$ ensures that:
		\begin{equation}
			\|\mathbf{R}^t - \mathbf{T}\|_F \leq \|\mathbf{Y} - \mathbf{T}\|_F = \|\mathbf{E}\|_F.
		\end{equation}
		
		Therefore, given that $\|\mathbf{E}\|_F \ll \|\mathbf{T}\|_F$, we have:
		\begin{equation}
			\mathbf{R}^t \approx \mathbf{T} \quad \text{and} \quad \text{rank}(\mathbf{R}^t) \leq r,
		\end{equation}
		which completes the proof of Lemma \ref{lemma2} .
	\end{IEEEproof}
	\section{Additional Experimental Details}
	\begin{table}[h]
		\setlength{\abovecaptionskip}{0.2cm} 
		\setlength{\belowcaptionskip}{0.2cm}
		\caption{General Information of The nine Multi-label datasets and four partial multi-label learning datasets.}
		\label{datasets}
		\setlength{\tabcolsep}{1mm} 
		\renewcommand{\arraystretch}{1}
		\resizebox{1\linewidth}{!}{
			\begin{tabular}{cccccc}
				\toprule
				Datasets &\#Instance &\#Dim &\#Class &avg.\#CLs &avg.\#GLs\\
				\hline
				Mirflickr &10433&100&7&3.35&1.77 \\ 
				Music$\_$emotion &6833&98&11&5.29&2.42  \\ 
				Music$\_$style &6839&98&10&6.04&1.44  \\ 
				YeastBP &6139&6139&127&5.93&5.54 \\ 
				\hline
				emotions &593&72&6&3, 4, 5 &1.86  \\ 
				birds &645&260&19&3, 4, 5&1.01  \\ 
				image &2000&294&5&2, 3, 4&1.23 \\ 
				scene &2407&294&6&3, 5&1.07 \\ 
				yeast &2417&103&14&7, 9&4.24 \\ 
				health &5000&612&32&7, 9&1.67 \\ 
				recreation &5000&606&22&7, 9, 11&1.42 \\ 
				arts &5000&462&26&5, 7, 9, 11&1.64  \\
				reference &5000&793&33&5, 7, 9, 11&1.17  \\
				\bottomrule
		\end{tabular}}
	\end{table}
	
	
	\begin{table}[t]
		{\centering
			\caption{The Average Rankings of Each Method Across All Five Metrics (Transposed)}
			\label{rank_transposed}
			\begin{tabular}{lcccccc}
				\toprule
				Method & Hl & Rl & Oe & Cov & Ap & avg \\
				\midrule
				PML-MA   & 2.0833 & 1.2500 & 1.2500 & 1.6000 & 1.1000 & \textbf{1.4567} \\
				PML-PLR  & 2.9167 & 2.2000 & 2.3333 & 2.1833 & 2.2333 & 2.3733 \\
				FBD-PML  & 5.2333 & 4.9167 & 4.5333 & 4.5000 & 4.5333 & 4.7433 \\
				PML-ND   & 3.4167 & 3.5333 & 3.4167 & 3.4667 & 2.9333 & 3.3533 \\
				P-LENFN  & 6.7167 & 6.2167 & 6.1500 & 6.2333 & 5.7167 & 6.2067 \\
				PAMB     & 5.8500 & 4.9333 & 6.3000 & 5.6667 & 5.7167 & 5.6933 \\
				PML-NI   & 5.3833 & 7.4667 & 6.8833 & 7.7000 & 7.0667 & 6.9000 \\
				PARTICLE & 7.6000 & 7.9500 & 7.2500 & 7.3000 & 8.0167 & 7.6233 \\
				PML-fp   & 5.8000 & 6.5333 & 6.8833 & 6.3500 & 7.6833 & 6.6500 \\
				\bottomrule
		\end{tabular}}
	\end{table}
	\subsection{Comparing Approaches description}
	{For all baseline methods, the hyperparameter settings are taken directly from their original publications to ensure fair comparison. The specific settings are listed below.}
	A brief supplementary description of each comparison algorithm is provided below:
	\begin{itemize}
		\item { PML-PLR (IJCAI'2025): A PML method that jointly extracts instance-level correlations from both candidate labels and features, and then uses them as reconstruction coefficients to reconstruct pseudo-labels, thereby disambiguating labels. [Parameter setting:$\alpha$=0.1, $\beta$=0.1, $\lambda3$=0.01].}
		\item {\texttt{}}{FBD-PML (Neural Networks'2025) : A PML method connects feature and label spaces via fuzzy confidence scores and employs manifold embedding to preserve structural consistency. [Parameter setting:$\lambda1$=0.01, $\lambda2$=0.01, $\lambda3$=10, $\lambda4$=0.1, $\lambda4$=1].}
		\item {\texttt{}} {PML-ND (TMM'2024): A PML method utilizes negative label information and noise feature information for candidate label disambiguation. [Parameter setting: the code comes with built-in parameter settings.].}
		\item P-LENFN (ACM MM'2024): A PML method utilizes both near and far neighbor information to improve label consistency, while incorporating nonlinear properties to enhance classifier. [Parameter setting:$\lambda1$=10, $\lambda2$=1, $\lambda3$=0.01, $\lambda4$=10].
		\item PAMB (TPAMI'2023): A PML technique employs ECOC for the generation of binary label sets and applies loss-weigted predictions to handle unseen instances effectively. [Parameter setting: $z = avg.\#CLs$, $L=100\log_2 (q)$].
		\item PML-NI (TPAMI'2023): A PML method is proposed for truth label prediction and noise label recognition by decomposing the prediction model matrix into independent components.
		[Parameter setting: $\lambda$=10, $\beta$=0.5, $\gamma$=0.5].
		\item PARTIAL (PAR-MAP and PAR-VLS) (TPAMI'2022): A PML method that iteratively refines candidate labels through label propagation and constructs distinct predictive models.
		[Parameter setting:thr=0.9, $\alpha$=0.95, k=10].
		\item {PML-fp (AAAI'2018): A PML method leverages feature prototype learning to assign confidence to each label and designs a corresponding weighted ranking loss function. [Parameter setting: $C_1$=1, $C_2$=3, $C_3$=10].}
	\end{itemize} 	
	
	\subsection{Dataset statistics}		
	Details of the experimental data set used in this paper are shown in Table \ref{datasets}. This includes the number of instances $(\# Instance)$, the number of dimensions $(\# Dim)$, the number of classes $(\# Class)$, the average number of candidate labels $(avg.\#CLs)$, and the average number of ground-truth labels $(avg.\#GLs)$. 
	
	\subsection{Experimental Results}
	
	Table \ref{rank_transposed} presents the average rankings of all compared methods across 
	five evaluation metrics (\textit{Hamming loss}, \textit{Ranking loss}, 
	\textit{One-error}, \textit{Coverage}, and \textit{Average precision}) 
	on 30 experimental configurations. The rankings are computed by ordering 
	methods' performance on each dataset-metric combination, then averaging 
	across all 30 cases. Lower values indicate better overall performance. 
	As shown, PML-MA achieves the best average ranking (1.4567), significantly 
	outperforming all baseline methods. These rankings serve as the basis 
	for the Friedman test and Nemenyi post-hoc analysis presented in the 
	main paper.
	
\end{document}